\documentclass[twoside]{article}
\usepackage[accepted]{aistats2020}

\usepackage{amsfonts}
\usepackage{latexsym}
\usepackage{hyperref}
\usepackage{enumerate}
\usepackage{amssymb}
\usepackage{amsmath}
\usepackage{float}
\usepackage[ruled,noend,noline,slide]{algorithm2e}

\usepackage{graphicx}
\usepackage{amsthm}
\usepackage[round]{natbib}

     \newcommand{\ind}{_{n,k}}
\newcommand{\inv}{^{-1}}
\newcommand{\inda}{_{1,k}}
\newcommand{\bE}{\mathbb{E}}
\newcommand{\bP}{\mathbb{P}}
\newcommand{\bK}{\mathbb{K}}
\newcommand{\I}{\mathbb{I}}
\newcommand{\som}{\mu_k^2}
\newcommand{\delmu}{\Delta_k\mu_k}

\newtheorem{assumption}{Assumption}
\newtheorem{definition}{Definition}
\newtheorem{proposition}{Proposition}
\newtheorem{corollary}{Corollary}
\newtheorem{claim}{Claim}

\newtheorem{remark}{Remark}
\newtheorem{lemma}{Lemma}
\newtheorem{theorem}{Theorem}

\usepackage[round]{natbib}

\begin{document}

%

%

\twocolumn[

\aistatstitle{Budget-Constrained Bandits over General Cost and Reward Distributions}

\aistatsauthor{ Semih Cayci \And Atilla Eryilmaz \And  R. Srikant }

\aistatsaddress{ ECE, The Ohio State University \And  ECE, The Ohio State University \And CSL and ECE, UIUC } 

]

\begin{abstract}
We consider a budget-constrained bandit problem where each arm pull incurs a random cost, and yields a random reward in return. The objective is to maximize the total expected reward under a budget constraint on the total cost. The model is general in the sense that it allows correlated and potentially heavy-tailed cost-reward pairs that can take on negative values as required by many applications. We show that if moments of order $(2+\gamma)$ for some $\gamma > 0$ exist for all cost-reward pairs, $O(\log B)$ regret is achievable for a budget $B>0$. In order to achieve tight regret bounds, we propose algorithms that exploit the correlation between the cost and reward of each arm by extracting the common information via linear minimum mean-square error estimation. We prove a regret lower bound for this problem, and show that the proposed algorithms achieve tight problem-dependent regret bounds, which are optimal up to a universal constant factor in the case of jointly Gaussian cost and reward pairs. 
\end{abstract}

\section{Introduction}

Multi-armed bandit problem (MAB) has been the prominent model for the exploration-and-exploitation dilemma since its introduction \citep{robbins1952some, lai1985asymptotically, berry1985bandit}. Due to the universality of the dilemma, bandit algorithms have found a broad area of applications from medical trials and dynamic pricing to ad allocation. As a common feature of all MAB instances, each action depletes a cost from a limited budget, and a random reward is obtained in return. In such a setting, the aim of the decision maker is to balance the exploration and exploitation at every step so as to maximize the cumulative reward until depleting the budget. In the classical MAB setting, each action is assumed to consume a known deterministic amount of resource, i.e., one time-slot. However, in many problems of interest, different tasks consume different and random amount of resources, which can be unbounded and potentially correlated with the reward. The applications of this extended setting include routing in communications and task scheduling in computing systems, where the controller sequentially makes a selection among multiple arms (alternative paths or task types) so as to maximize the total reward (i.e., throughput) within a given time budget. In these applications, the cost (i.e., completion time) and reward of each arm pull can be potentially correlated and heavy-tailed \citep{harchol2000task, jelenkovic2013characterizing}.

In this paper, we investigate the unique dynamics of this extended budget-constrained bandit setting with general cost and reward distributions. Unlike the classical stochastic MAB problem, each action incurs a random cost and yields a random reward in our model. Under a budget constraint $B$, the objective of the controller is to maximize the expected cumulative reward until the total cost exceeds the budget. As we will see, the correlation and variability of the cost-reward pairs can have a substantial impact on the performance in this bandit setting, which we incorporate in the design of learning algorithms for near-optimal performance. Many of our results are obtained for a very general setting where the cost and reward can be correlated and heavy-tailed, but sharper results are presented for some interesting special cases.

\subsection{Main Contributions}
The main objective in this paper is to design efficient algorithms that achieve provably tight regret bounds in an extended setting of correlated and potentially heavy-tailed cost and reward. Our main contributions are as follows:

\begin{enumerate}
    \item \textbf{Exploiting the correlation:} One of the key contributions in this work is to use a linear minimum mean square (LMMSE) estimator to extract and exploit the correlation between the cost and reward of an arm (see Section \ref{subsec:ucb-b1}). Furthermore, we incorporate the effect of variability in cost-reward pairs through variance. Consequently, we achieve provably tight problem-dependent regret bounds in an extended setting of unbounded cost and reward.
    \item \textbf{Extension to unbounded cost and reward:} We develop novel design and analysis methods for the setting of unbounded and potentially heavy-tailed cost and reward pairs, and show that $O\big(\log(B)\big)$ regret is achievable if moments of order $2+\gamma$ exist for some $\gamma>0$ for all cost and reward pairs (see Section \ref{sec:median-alg}). 
    \item \textbf{Regret lower bounds:} We establish a regret lower bound for the budget-constrained bandit problem (see Section \ref{sec:regret-lb}). By using this result, we obtain explicit regret lower bounds for jointly Gaussian cost-reward distributions. Consequently, we prove that the algorithms we propose in this paper achieve tight regret bounds, which are optimal up to a constant factor in the case of jointly Gaussian cost and reward.
\end{enumerate}

 \subsection{Related Work}
 The classical stochastic multi-armed bandit problem, which is a specific case of the model we study in this paper, has been extensively studied in the literature. For detailed discussion on the basic model, we refer to \citep{bubeck2012regret, berry1985bandit}.
 
 The budget-constrained MAB problem and its variants were investigated in a variety of papers. In \citep{tran2012knapsack} and \citep{combes2015bandits}, budget-constrained multi-armed bandit problem is investigated where each arm pull incurs an arm-dependent and deterministic cost. In \citep{guha2009multi}, the budgeted-bandit problem with deterministic costs is investigated from a Bayesian perspective, and constant-factor approximation algorithms are proposed. In \citep{gyorgy2007continuous}, the continuous-time extension of the MAB problem with side information is investigated, which is an early example for the budget-constrained bandit problem. In \citep{badanidiyuru2013bandits, agrawal2014bandits}, the bandit problem under multiple budget constraints is examined, and problem-independent regret bounds of order $\tilde{O}(\sqrt{B})$ are obtained. Bandits with knapsacks have been extended to other bandit settings \citep{agrawal2016linear, badanidiyuru2014resourceful, sankararaman2017combinatorial, ding2013multi}. In \citep{xia2015thompson,xia2016budgeted}, the budget-constrained MAB problem is explored in a similar setting to ours. In these works, the cost and reward of each arm are supported in $[0,1]$, and the correlation between them is not exploited. In \citep{cayci2019learning}, the authors consider a variation of the budget-constrained bandit problem where the controller has the option to interrupt an ongoing cycle for a faster alternative. The interruption mechanism brings significantly different dynamics to the problem that is investigated in this paper.

 Bandits with heavy-tailed reward distributions are considered in \citep{liu2011multi, bubeck2013bandits}. These papers are still in the scope of the classical MAB setting: the budget is consumed deterministically at rate 1 by each action, so the dynamics of the random resource consumption with heterogeneous statistics are not included in the model.  
 


\section{System Setup}\label{sec:system-setup}
In this paper, we consider a bandit problem with $K$ arms. The set of arms is denoted by $\bK=\{1,2,\ldots, K\}$. Each arm $k\in\bK$ is described by a two-dimensional random process $\{(X_{n,k}, R_{n,k}):n\geq 1\}$ that is independent from other arms. If arm $k$ is chosen at $n$-th epoch, it incurs a cost of $X_{n,k}$ and yields a reward of $R_{n,k}$, where both are learned via a bandit feedback only after the decision is made. The controller has a cost budget $B>0$, and tries to maximize the expected cumulative reward it receives by sampling the arms wisely under this budget constraint. 

The pair $(X\ind,R\ind)$ is assumed to be independent and identically distributed over $n$, but the cost $X\ind$ and reward $R\ind$ can be positively correlated. We allow $X\ind$ to take on negative values, but the drift is assumed to be positive, i.e., there exists $\mu_*>0$ such that $\bE[X\ind] \geq \mu_* > 0$ for all $k$.

Let $\pi$ be an algorithm that yields a sequence of arm pulls $\{I_n^\pi\in \bK:n\geq 1\}$. Under $\pi$, the history until epoch $n$ is the following filtration:
\begin{equation}
    \mathcal{F}_n^\pi=\sigma(\{(X_{j,k},R_{j,k}):I_j^\pi = k, 1\leq j \leq n\}),
\end{equation}
\noindent where $\sigma(X)$ denotes the sigma-field of a random variable $X$. We call an algorithm $\pi$ admissible if $\pi$ is non-anticipating, i.e., $\{I_n^\pi = k\}\in\mathcal{F}_{n-1}^\pi$ for all $k, n$. The set of all admissible policies is denoted as $\Pi$. 

The total cost incurred in $n$ epochs under an admissible policy $\pi\in\Pi$ is a controlled random walk which is defined as $S_n^\pi = \sum_{i=1}^nX_{i, I_i^\pi}.$ The arm pulling process under an algorithm $\pi$ continues until the budget $B$ is depleted. We assume that the reward corresponding to the final epoch during which the budget is depleted is gathered by the controller. Thus, the total number of pulls under $\pi$ is defined as follows:
\begin{equation}
    N_\pi(B) = \inf\Big\{n: S_n^\pi > B\Big\}.
\end{equation}
\noindent Note that the total number of pulls $N_\pi(B)$ is a stopping time adapted to the filtration $\{(\mathcal{F}_t^\pi): t \geq 0\}$. With these definitions, the cumulative reward under a policy $\pi$ can be written as follows:
\begin{equation}
    \label{eqn:reward}
    {\tt REW}_\pi(B) = \sum_{i=1}^{N_\pi(B)}R_{i, I_i^\pi}.
\end{equation}

The objective in this paper is to design algorithms that achieve maximum $\bE[{\tt REW}_\pi(B)]$, or equivalently minimum regret, which is defined as follows:
\begin{equation}
    \label{eqn:regret-def}
    Reg_\pi(B) = \bE[{\tt REW}_{\pi^{\tt opt}}(B)] - \bE[{\tt REW}_\pi(B)],
\end{equation}
\noindent where $\pi^{\tt opt}(B)$ denotes the optimal policy: $$\pi^{\tt opt}(B) \in \underset{\pi^\prime\in\Pi}{\arg\max}~\bE[{\tt REW}_{\pi^\prime}(B)],$$ for any $B > 0$.

In the following section, we investigate the optimal policy that maximizes the expected cumulative reward when all arm distributions are known, and provide low-complexity approximations that have desirable performance characteristics.

\section{Approximations of the Oracle}
The optimization problem described in Section \ref{sec:system-setup} is a variant of the unbounded knapsack problem, and it is known that similar stochastic control problems are PSPACE-hard  \citep{badanidiyuru2013bandits, papadimitriou1999complexity}. In order to find a tractable benchmark, we will consider approximation algorithms with provably good performance in this section.

The main quantity of interest will be the reward rate, which is defined as follows:
\begin{equation}
    r_k = \frac{\bE[R_{1, k}]}{\bE[X_{1,k}]},~k\in \bK.
\end{equation}
\noindent Intuitively, if arm $k$ is chosen persistently until the budget $B>0$ is depleted, the cumulative reward becomes $r_kB+o(B)$ as $B\rightarrow \infty$. The additive $o(B)$ term is $O(1)$ if $\bE[(X_{1,k}^+)^2]<\infty$ by Lorden's inequality \citep{asmussen2008applied}. Hence, pulling the arm with the highest reward rate is a logical choice.

In the following, we prove that the optimality gap is $O(1)$ under mild moment conditions, which covers the case of heavy-tailed cost-reward pairs.


\begin{definition}[Optimal Static Algorithm]
    Let $k^*$ be the arm with the highest reward rate: $$k^* \in \underset{k\in\bK}{\arg\max} ~r_k.$$ The optimal static policy, denoted by $\pi^*$, pulls $k^*$ until the budget is depleted: $I_n^{\pi^*} = k^*$ for all $n\leq N_{\pi^*}(B)$.
\end{definition}

The main result of this section is the following proposition, which implies that $\pi^*$ is a plausible approximation algorithm for $\pi^{\tt opt}(B)$ for all $B>0$ under mild moment conditions.

\begin{assumption}\label{assn:moment-basic}
    There exists $\gamma > 0$ such that $\bE[(X_{1,k}^+)^{2+\gamma}] < \infty$ for all $k\in\bK$.
\end{assumption}

\begin{proposition}[Optimality Gap for $\pi^*$]\label{prop:optimality-gap}
    Under Assumption \ref{assn:moment-basic}, there exists a constant $$G^\star = G^\star\Big(\min\limits_k\bE[X_{1,k}],\max\limits_k Var(X_{1,k})\Big) < \infty,$$ independent of $B$ such that the following holds:
    \begin{equation}
        \max\limits_{\pi\in\Pi}~\bE[{\tt REW}{\pi}(B)]-\bE[{\tt REW}{\pi^*}(B)] \leq G^\star,
    \end{equation}
    \noindent for any $B > 0$. Consequently, $\pi^*$ is asymptotically optimal as $B\rightarrow \infty$.
\end{proposition}
\begin{proof}
   The proof of Proposition \ref{prop:optimality-gap} is based on tools from stochastic control, and is given in Appendix A.
\end{proof}
\noindent Proposition \ref{prop:optimality-gap} implies that the optimality gap of the optimal static policy is a constant with respect to the budget $B$, which depends on the first- and second-order moments of the cost. This extends the result presented in \citep{xia2016budgeted} for bounded and strictly positive costs to unbounded costs with positive drift that can take on negative values. Also, for small $B$
values, there can be dynamic policies that outperform this simple static policy \citep{dean2004approximating}. However, the optimality gap is still $O(1)$ for these dynamic policies, therefore we consider $\pi^*$ for its simplicity and efficiency.

Now that we have an accurate approximation for the oracle, we propose the first and basic algorithms that assume the knowledge of second-order moments.

\section{Algorithms for Known Second-Order Moments}\label{sec:alg-known-var}
In this section, we will assume that the second-order moments of all cost-reward pairs are known by the decision maker. First, in Section \ref{subsec:ucb-b1}, we will consider the case $(X\ind,R\ind)$ are jointly Gaussian, and propose a learning algorithm that achieves tight regret bound on the order of $O(\log(B))$ by using the correlation information. Then, in Section \ref{sec:median-alg}, we will study the general case where the cost and reward can be unbounded and potentially heavy-tailed, and propose algorithms that achieve the same regret bounds (up to a constant) as the sub-Gaussian case.

The following proposition provides a basis for the algorithm design and analysis throughout the paper.
\subsection{Preliminaries: Rate Estimation}
Let $\theta = (\theta_1,\theta_2)\in\mathbb{R}^2$ be a pair of unknown constants for which $r=\frac{\theta_2}{\theta_1}$ is to be estimated. The following proposition yields a useful device to obtain concentration results for r from concentration results for $\theta_1$ and $\theta_2$ for this estimation procedure.

\begin{proposition}[Rate Estimation]\label{prop:rate-est}
    Let $\widehat{\theta}_1$ and $\widehat{\theta}_2$ be estimators for $\theta_1>0,\theta_2\geq 0$, respectively. If \begin{equation}
        \eta\in\big(0,\frac{\theta_1(\lambda-1)}{\lambda}\big),
        \label{eqn:st-condition}
    \end{equation} \noindent for some $\lambda > 1$, then we have the following result:
    \begin{multline*}
    \bP\Big(|r-\frac{\hat{\theta}_2}{\hat{\theta}_1}| > \frac{\lambda(\epsilon+r\eta)}{\theta_1}\Big) \leq \bP(|\hat{\theta}_1-\theta_1|>\eta)\\+\bP(|\hat{\theta}_2-\theta_2|>\epsilon).
    \end{multline*}
\end{proposition}
\noindent Therefore, if $\hat{\theta}_1$ and $\hat{\theta}_2$ both achieve exponential convergence rate, then $\frac{\hat{\theta}_2}{\hat{\theta}_1}$ converges to $r$ exponentially fast. 
   The intuition behind the proposition is illustrated in Figure \ref{fig:rate-est}.
\begin{figure}\label{fig:rate-est}
    \begin{center}
        \includegraphics[scale=0.4]{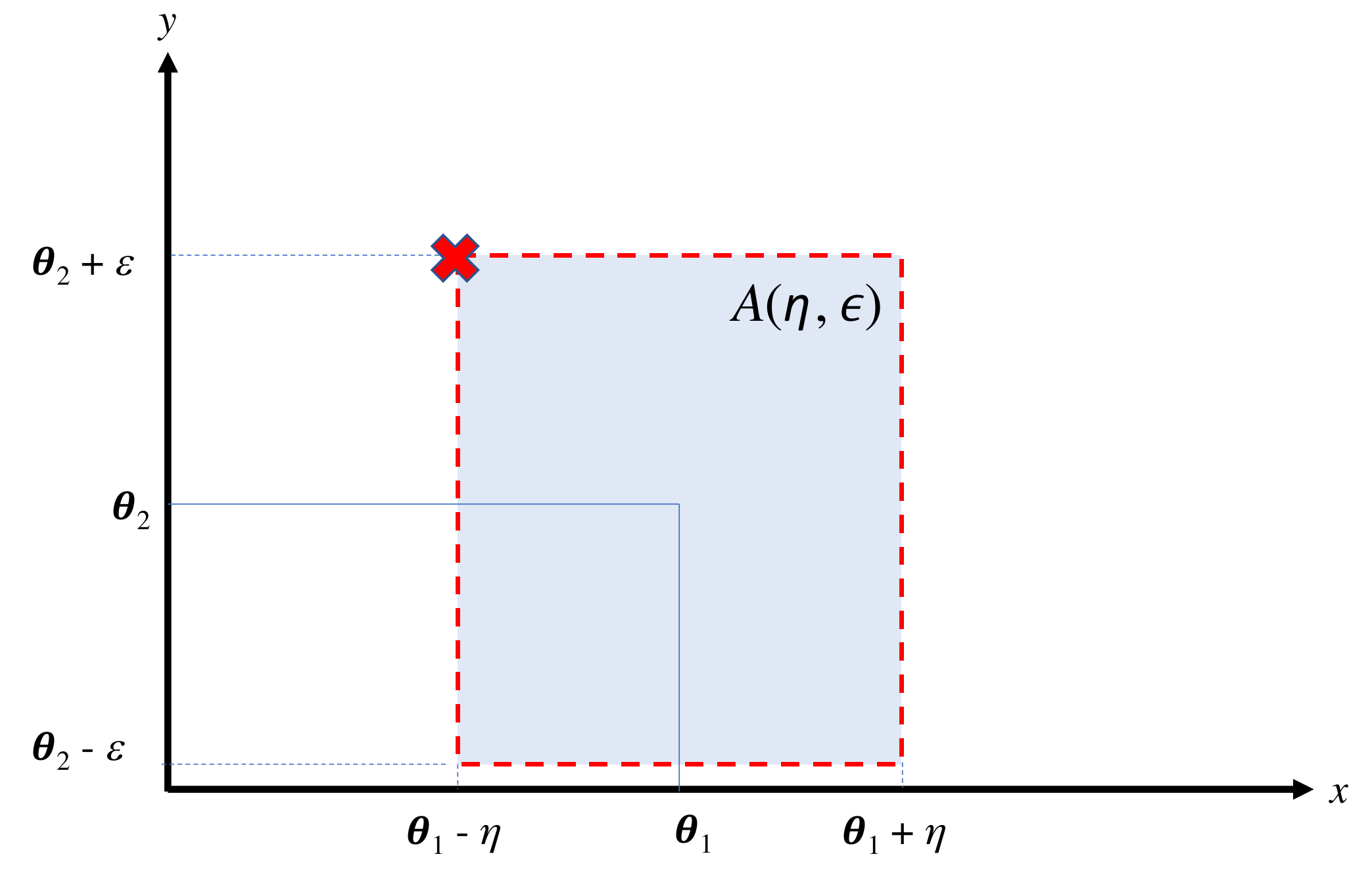}
    \end{center}
    \caption{If $(\hat{\theta}_1,\hat{\theta}_2)$ is in the high-probability set $A(\eta,\epsilon)$, then the maximum deviation of $\hat{r}=\frac{\hat{\theta}_2}{\hat{\theta}_1}$ from $r$ is $\frac{\lambda(\epsilon+r\eta)}{\theta_1}$, and it is achieved at the marked corner.}
\end{figure}

\begin{remark}[Stability of the rate estimator]\label{rem:rate-est}\normalfont 
The condition $\eta < \theta_1$, i.e., sufficient concentration of the estimator around the true parameter $\theta_1$, is crucial for Proposition \ref{prop:rate-est}. Note that if the variability of the mean estimator is high and thus $A(\eta,\epsilon)$ intersects with the $y$-axis, then the above bound is useless as $\hat{r}$ can have arbitrarily large deviations from $r$. 

\end{remark}

In the following, we propose algorithms under the assumption that the second-order moments for each arm $k$ is known by the controller.

\subsection{Sub-Gaussian Case: Algorithm {\tt UCB-B1}}\label{subsec:ucb-b1}
The main idea behind {\tt UCB-B1} is to use an upper confidence bound for the reward rate $r_k$. Let $T_k(n)$ be the number of pulls for arm $k$ in the first $n$ stages and $\widehat{r}_{k,n}=\frac{\max\{0,\widehat{\bE}_n[R_k]\}}{\max\{b, \widehat{\bE}_n[X_k]\}}$ where 
\begin{align*}
    \widehat{\bE}_n[X_k] &= \frac{1}{T_k(n)}\sum_{i=1}^n\I\{I_i=k\}X_{i,k}, \\
    \widehat{\bE}_n[R_k] &= \frac{1}{T_k(n)}\sum_{i=1}^n\I\{I_i=k\}R_{i,k},
\end{align*}
\noindent and $b \leq \bE[X\inda]/2$ for all $k$. Instead of estimating $\bE[X\inda]$ and $\bE[R\inda]$ separately from the samples of $(X\ind,R\ind)$, the correlation between $X\ind$ and $R\ind$ can be exploited to tighten the upper confidence bound for $r_k$. This is achieved by estimating $R\ind$ by a linear estimator $\omega X\ind$ so as to minimize $Var(R\ind-\omega X\ind)$. Let \begin{align}\label{eqn:mmse}
    V(X_{1,k},R_{1,k}) = \min\limits_{\omega\in\mathbb{R}}~Var(R_{1,k}-\omega X_{1,k}).
\end{align}
\noindent If $Var(X_{n,k})>0$, we have:
\begin{align}\label{eqn:lmmse}
\begin{aligned}
        \omega_k &= \underset{\omega\in\mathbb{R}}{\arg\min}~Var(R_{1,k}-\omega X_{1,k}),\\
        &= \frac{Cov(X_{1,k},R_{1,k})}{Var(X_{1,k})},
\end{aligned}
\end{align}
\noindent by the orthogonality principle \citep{poor2013introduction}, and the optimal value of the objective is given by: $$V(X_{1,k},R_{1,k}) = Var(R_{1,k})-\omega_k^2Var(X_{1,k}).$$ If $Var(X_{n,k})=0$, we have $V(X_{1,k},R_{1,k}) = Var(R_{1,k})$. This implies that $\omega_k$ and $V$ can be computed from the second-order moments of $(X\ind,R\ind)$, which are assumed to be given in this section. For simplicity, we assume $\omega_k\leq r_k$ for all $k$ throughout the paper.

For non-negative $(M_X,M_R,L)$ that will be specified later, let
\begin{align*}
    \epsilon_{k,n}^{\tt B}&=\frac{2\alpha M_R\log(n)}{3T_k(n)}+\sqrt{L\alpha\frac{V(X_{1,k},R_{1,k})\log(n)}{T_k(n)}}, \\
    \eta_{k,n}^{\tt B}&=\frac{2\alpha M_X\log(n)}{3T_k(n)}+\sqrt{L\alpha\frac{Var(X_{1,k})\log(n)}{T_k(n)}}.
\end{align*}



\noindent Then, if $S_n^\pi < B$, i.e., there is a remaining budget, then the {\tt UCB-B1} Algorithm pulls an arm at stage $n+1$ according to:
\begin{equation*}
    I_{n+1}\in\arg\max\limits_k~\Big\{\widehat{r}_{k,n} + \widehat{c}_{k,n}^{\tt B1}  \Big\},
\end{equation*}
\noindent where $$\widehat{c}_{k,n}^{\tt B} =  
      1.4\frac{\epsilon_{k,n}^{\tt B}+(\widehat{r}_{k,n}-\omega_k)\eta_{k,n}^{\tt B}}{\big(\widehat{\bE}_n[X_k])^+}$$ if the stability condition \eqref{eqn:st-condition} holds for $\eta = \eta_{k,n}^{\tt B}$ and $\lambda = 1.28$, and $\widehat{c}_{k,n}^{\tt B} =\infty$ otherwise.


The regret performance of {\tt UCB-B1} is presented in the following theorem.
\begin{theorem}[Regret Upper Bound for {\tt UCB-B1}]\label{thm:ucb-b1}
Let $\Delta_k = r^*-r_k$, $\lambda = 1.28$, 
\begin{equation}\label{eqn:sigma-k}
    \sigma_k^2 = 
      V(X_{1,k},R_{1,k})+(r^*-\omega_k)^2Var(X_{1,k}),
\end{equation}
\noindent for all $k\in\bK$ and recall that $\mu_*=\min\limits_k~\bE[X_{1,k}]$.
\begin{enumerate}
    \item \textbf{Bounded Cost and Reward:} If $|X_{1,k}|\leq M_X$, $|R_{1,k}|\leq M_R$ a.s., $\alpha > 2$ and $L=2$, then the regret under {\tt UCB-B1} is upper bounded as:
    \begin{equation*}
        Reg_{\pi^{\tt B1}}(B) \leq \alpha\sum_{k:\Delta_k>0}\log\Big(\frac{2B}{\mu_*}\Big)C_k^{\tt B1}+O(1),
    \end{equation*}
    \noindent for some constant $\zeta > 1$ where $M_k = M_R+r_kM_X$ and 
    \begin{equation*}
        C_k^{\tt B1} =\frac{42\sigma_{k}^2}{\Delta_k\bE[X_{1,k}]}+ 42M_k+21M_X\Delta_k,
    \end{equation*}
    \noindent for all $k$.
    \item \textbf{Jointly Gaussian Cost and Reward:} Let $(X\ind,R\ind)$ be jointly Gaussian with known second-order moments. Then, {\tt UCB-B1} with $\alpha > 2$, $M_X=M_R=0$ and $L=\frac{1}{2}$ yields the following regret bound:
    \begin{equation*}\label{eqn:reg-ucb-b1}
        Reg_{\pi^{\tt B1}}(B) \leq \alpha\sum_{k:\Delta_k>0}\log\Big(\frac{2B}{\mu_*}\Big)\frac{11\sigma_k^2}{\Delta_k\bE[X_{1,k}]} + O(1),
    \end{equation*}
    \noindent where $\sigma_k$ is defined in \eqref{eqn:sigma-k}.
\end{enumerate}

\end{theorem}
\begin{proof}
    The detailed proof, which will provide basis for the analysis of other algorithms proposed in this work, can be found in Appendix C. Note that the total reward is a controlled and stopped random walk with potentially unbounded support. Thus, the regret analysis requires new methods from the theory of martingales and stopped random walks. As such, we follow a proof strategy based on establishing a high-probability upper bound for $N_\pi(B)$, which can be found in Appendix B. 
\end{proof}

\subsection{Heavy-Tailed Case: Algorithm {\tt UCB-M1}}\label{sec:median-alg}
In this subsection, we design a general algorithm that achieves the regret in the sub-Gaussian case (up to a constant) under the mild moment condition that $\bE[(X\inda^+)^{2+\gamma}]<\infty$ for all $k$.

The empirical mean estimator played a central role in the design of the {\tt UCB-B1} Algorithm for sub-Gaussian distributions, which is proved to achieve $O(\log(B))$ regret. However, if we consider heavy-tailed distributions, the empirical mean estimator fails to achieve exponential convergence rate due to the frequent outliers \citep{bubeck2013bandits}. The median-based estimators, introduced in \citep{nemirovsky1983problem} provide an elegant method to boost the convergence speed in mean estimation. The idea of boosting the confidence of weak independent estimators by taking the median was extended to general point estimation problems (beyond the mean estimation) in \citep{minsker2015geometric}. In the following, we will use a variation of this method in the design of median-based rate estimators.

Consider arm $k\in\bK$ at stage $n$. For $$m = \lfloor 3.5\alpha\log(n)\rfloor+1,$$ we partition the observed samples $\{(X_{i,k},R_{i,k}):I_i=k,~1\leq i \leq n\}$ into index sets $G_1,G_2,\ldots,G_m$ of size $\lfloor T_k(n)/m \rfloor$ each. Then, for each $j\in\{1,2,\ldots,m\}$, let $\tilde{r}_{k, G_j}=\frac{\max\{\widehat{\bE}_{G_j}[R_k],0\}}{\max\{\widehat{\bE}_{G_j}[X_k],b\}}$ where $b \leq \bE[X\inda]/2$, and
\begin{align*}
    \widehat{\bE}_{G_j}[X_k] = \sum_{i\in G_j}\frac{X_{i,k}}{|G_j|},\hskip 0.5cm \widehat{\bE}_{G_j}[R_k] = \sum_{i\in G_j}\frac{R_{i,k}}{|G_j|}.
\end{align*}
\noindent The median-based rate estimator for arm $k$ at stage $n$ is thus $$\overline{r}_{k,n} = \underset{1\leq j \leq m}{\mbox{median}}~\tilde{r}_{k,G_j}.$$
The deviations in the cost and reward are as follows:
\begin{align*}
    \epsilon_{k,n}^{\tt M}&=11\sqrt{\alpha\frac{V(X_{1,k},R_{1,k})\log(n)}{T_k(n)}}, \\
    \eta_{k,n}^{\tt M}&=11\sqrt{\alpha\frac{Var(X_{1,k})\log(n)}{T_k(n)}}.
\end{align*}

Therefore, the decision at stage $(n+1)$ under {\tt UCB-M1} is as follows:
\begin{align}
    I_{n+1}\in\arg\max\limits_k\Big\{\overline{r}_{k,n} + \widehat{c}_{k,n}^{\tt M}\Big\}
\end{align}
\noindent where $$\widehat{c}_{k,n}^{\tt M} = \frac{2\sqrt{2} \big(\epsilon_{k,n}^{\tt M}+(\overline{r}_{k,n}-\omega_k)\eta_{k,n}^{\tt M}}{\Big(\underset{1\leq j \leq m}{\mbox{median}}~\widehat{\bE}_{G_j}[X_k]\Big)^+},$$ if the condition \eqref{eqn:st-condition} is satisfied for $\eta = \underset{1\leq j \leq m}{\mbox{median}}~\widehat{\bE}_{G_j}[X_k]$ and $\lambda = 1.28$.

For {\tt UCB-M1}, we have the following regret upper bound.
\begin{theorem}[Regret Upper Bound for {\tt UCB-M1}]\label{thm:ucb-m1}
If the following moment conditions hold: 
\begin{itemize}
    \item $\bE[(X\inda^+)^{2+\gamma}]<\infty$, for all $k$,
    \item $Var(R\inda)<\infty,$ for all $k$,
\end{itemize}
then the regret under {\tt UCB-M1} satisfies the following upper bound:
\begin{equation}\label{eqn:reg-ucb-m2}
        Reg_{\pi^{\tt M1}}(B) \leq \alpha\sum_{k:\Delta_k>0}\log\Big(\frac{2B}{\mu_*}\Big)\frac{C\sigma_{k}^2}{\Delta_k\bE[X_{1,k}]} + O(1),
    \end{equation}
    \noindent where $\sigma_k$ is as defined in \eqref{eqn:sigma-k} and $C>0$ is a constant.
\end{theorem}

\begin{proof}
    The proof uses tools from the theory of martingales and stopped random walks, and can be found in Appendix B and Appendix C.  
\end{proof}

\begin{remark}\normalfont 
We have the following observations from Theorem \ref{thm:ucb-b1} and \ref{thm:ucb-m1}:
\begin{itemize}
    \item If $Var(X_{1,k})\downarrow 0$ and $\bE[X_{1,k}]=1$, the regret upper bounds match with the existing regret bounds for the stochastic bandit problem.
    \item Note that for positively correlated $X\ind$ and $R\ind$, one can ignore the correlation and use an upper confidence bound based on the separate estimation of $X\ind$ and $R\ind$. From Theorem \ref{thm:ucb-b1}, it can be observed that this scheme leads to a loss of $O\big(\sum_k Cov(X\inda,R\inda)\big)$. Moreover, as it will be seen in the next section, this is nearly the best way of exploiting the correlation in the case of jointly Gaussian cost and reward pairs.
    
    \item The {\tt UCB-M1} Algorithm achieves the same regret upper bound as the {\tt UCB-B1} Algorithm up to a constant with much less moment assumptions: while {\tt UCB-B1} requires sub-Gaussianity, {\tt UCB-M1} requires only existence of moments of order $(2+\gamma)$ for some $\gamma>0$ for the costs, and second-order moments for the rewards. However, the constant that multiplies the $O(\log B)$ term is much higher in {\tt UCB-M1} than {\tt UCB-B1}, which can be viewed as the cost of generality.
    
    \item If the cost is deterministic, i.e., $Var(X\inda)=0$, then the regret is monotonically decreasing in $\Delta_k$ as $O\Big(\frac{\log B}{\Delta_k}\Big)$ for each arm $k$. However, for random costs, since $r^*=r_k+\Delta_k$, the regret bounds have an additive term scaling linearly in $\Delta_k$ as $O\Big(\log\Big(\frac{2B}{\mu_*}\Big)\sum_k \frac{Var(X_{1,k})}{\bE[X_{1,k}]}\Delta_k\Big)$, which might seem strange at first since the separability of a suboptimal arm $k$ increases with its corresponding $\Delta_k$. This is a unique phenomenon observed in the case of stochastic costs: recall from Remark \ref{rem:rate-est} that the rate estimator is unstable when the confidence interval for the estimation of $\bE[X_{1,k}]$ is large, and thus it incurs $\bE[X\inda]\Delta_k$ regret per pull since rate estimation is unreliable. As it will be seen in Corollary \ref{cor:gaussian}, the same term appears with the same coefficient in the regret lower bound for jointly Gaussian cost-reward pairs, which implies that it is inevitable at least in that case.
 \end{itemize}
\end{remark}

\section{Regret Lower Bound for Admissible Policies}\label{sec:regret-lb}
In this section, we will propose regret lower bounds for the budget-constrained bandit problem based on \citep{lai1985asymptotically}. In the specific case of jointly Gaussian cost-reward pairs, we can determine a lower bound explicitly, which provides useful insight about the impact of variability and correlation on the regret. 


In order to establish a regret lower bound, assume that the joint distribution of $\{(X\ind, R\ind):n\geq 1\}$ is parametrized by $\theta_k\in\Theta_k$ for some parameter space $\Theta_k$, i.e., $(X\ind,R\ind)\sim P_{\theta_k}$. For any $k\in \bK$ and $\theta\in\Theta_k$, let $r_k(\theta) = \frac{\bE_\theta[R\inda]}{\bE_\theta[X\inda]}$ be the reward rate (i.e., reward per unit cost). Furthermore, for a given bandit instance $\vec{\theta}=(\theta_1,\theta_2,\ldots,\theta_K)$, let $r^*=\max\limits_k r_k(\theta_k)$ be the optimal reward rate, and $\Delta_k=r^*-r_k(\theta_k)$. For admissible policies, we have the following regret lower bound, which is an extension of Lai-Robbins style regret lower bounds for the stochastic bandit problem \citep{lai1985asymptotically, burnetas1996optimal}.

\begin{theorem}[Regret Lower Bound]\label{thm:regret-lb}
Suppose that $\bE[(X_{1,k})^{2+\gamma}]<\infty$ for some $\gamma>0$ and $Var(R\inda)<\infty$ hold for all $k$. Assume that the following conditions are satisfied by $P_{k,\theta}$ for any $k$:
\begin{enumerate}
    \item If $r_k(\theta_1) > r_k(\theta_2)$, then $D(P_{k,\theta_2}||P_{k,\theta_1})<\infty$,
    \item (Denseness) $r_k(\Theta_k) = \{r_k(\theta):\theta\in\Theta_k\}$ is dense,
    \item (Continuity) $\theta\mapsto D(P_{k,\theta_k}||P_{k,\theta})$ is a continuous mapping.
\end{enumerate}
For a given bandit instance $\vec{\theta} = (\theta_1,\theta_2,\ldots,\theta_K)$, if $\pi\in\Pi$ is a policy such that $\bE[T_k^\pi(n)] = o(n^\alpha)$ for any $\alpha > 0$ and $k$ such that $r_k(\theta_k) < r^*$, then we have the following lower bound:
\begin{equation}
    \underset{B\rightarrow\infty}{\lim\inf}~\frac{Reg_\pi(B)}{\log(B)} \geq \frac{1}{2}\sum_{k:\Delta_k>0}\frac{\bE[X\inda]\Delta_k}{D_k^\star},
\end{equation}
\noindent where $D_k^\star$ is the solution to the following optimization problem:
\begin{equation*}
D_k^\star = \min_{\theta\in\Theta_k} D(P_{k,\theta_k}||P_{k,\theta})
\mbox{ subject to } r_k(\theta) \geq r^*.
\end{equation*}
\end{theorem}
\begin{proof}
    The proof can be found in Appendix E.
\end{proof}

The regret lower bound has an explicit form if the cost and reward distributions of each arm is jointly Gaussian with a known covariance matrix.
\begin{corollary}[Jointly Gaussian Cost and Reward]\label{cor:gaussian}
    Let $(X\ind,R\ind)$ be jointly Gaussian:
    \begin{equation*}
        (X_{n,k},R_{n,k})\sim \mathcal{N}(\mu_k,\Sigma_k),
    \end{equation*}
    \noindent for all $k\in\bK$ where $\mu_k = \big(\bE[X\ind],\bE[R\ind]\big)$ and $$\Sigma_k = \begin{pmatrix}
Var(X\ind)&Cov(X\ind,R\ind)\\
Cov(X\ind,R\ind)&Var(R\ind)\\
\end{pmatrix}.$$ If $\Sigma_k$ is known and $\mu_k$ is unknown by the controller for all $k\in\bK$, we have the following regret lower bound for the Gaussian case:
\begin{equation}\label{eqn:lower-bd-gaussian}
    \underset{B\rightarrow\infty}{\lim\inf}~\frac{Reg_\pi(B)}{\log(B)} \geq \sum_{k:\Delta_k>0}\frac{\sigma_k^2}{\bE[X\inda]\Delta_k},
\end{equation}
\noindent where $\sigma_k^2$ is defined in \eqref{eqn:sigma-k}.
\end{corollary}
\begin{proof}
    For known $\Sigma_k$, we have $D_k^\star = \frac{(\bE[X\inda]\Delta_k)^2}{2\sigma_k^2}$ for $\theta_k=\mu_k$ and $\Theta_k=\mathbb{R}_+^2$. Using this in Theorem \ref{thm:regret-lb} yields the result.
\end{proof}

\begin{remark}[Optimality of {\tt UCB-B1} and {\tt UCB-M1}]
Comparing \eqref{eqn:reg-ucb-b1} and \eqref{eqn:reg-ucb-m2} with \eqref{eqn:lower-bd-gaussian}, we can deduce that {\tt UCB-B1} and {\tt UCB-M1} achieve optimal regret up to a universal constant for the case of jointly Gaussian cost and reward pairs with known covariance matrix.
\end{remark}

\section{Algorithms for Unknown Second-Order Moments}\label{sec:alg-unknown-var}
In Section \ref{sec:alg-known-var}, we proposed algorithms under the assumption that the second-order moments are known for each arm $k$. However, in practice, these second-order moments are unknown, and therefore to be estimated from the samples collected via bandit feedback. In this section, we will propose algorithms that use these second-order moment estimates to achieve tight regret bounds. 


The general strategy in the development of the algorithms in this section is to use empirical estimates for the second-order moments that appear in {\tt UCB-B1} as a surrogate. 

\subsection{Bounded and Uncorrelated Cost and Reward: {\tt UCB-B2}}
For clarity, we first consider the case $X\ind$ and $R\ind$ are uncorrelated for all $k $ and $X\ind\in[0,M_X]$ and $R\ind\in[0,M_R]$ almost surely for known $M_X,M_R>0$. In this case, we will propose an algorithm based on a variant of the empirical Bernstein inequality, which was introduced in \citep{audibert2009exploration}.

For any $k$, let the variance estimate $\widehat{V}_{k,n}(X_k)$ be defined as follows: $$\widehat{V}_{k,n}(X_k) = \frac{1}{T_k(n)}\sum_{i=1}^n\I\{I_i=k\}\big(X_{i,k}-\widehat{\bE}_n[X\inda]\big)^2,$$ where $\widehat{\bE}_n[X_k]$ is the empirical mean of the observations up to epoch $n$.

The bias terms in {\tt UCB-B2} are defined as follows:
\begin{flalign*}
    \epsilon_{k,n}^{\tt B2}&=\sqrt{\frac{2\widehat{V}_{k,n}(R_k)\log(n^\alpha)}{T_k(n)}}+\frac{3M_R\log(n^\alpha)}{T_k(n)}, \\
    \eta_{k,n}^{\tt B2}&=\sqrt{\frac{2\widehat{V}_{k,n}(X_k)\log(n^\alpha)}{T_k(n)}}+\frac{3 M_X\log(n^\alpha)}{T_k(n)}.
\end{flalign*}
\noindent Let $\widehat{r}_{k,n}$ be the empirical reward rate estimator in Section \ref{subsec:ucb-b1}, and
\begin{equation}
    \widehat{c}_{k,n}^{\tt B2} = 1.4\frac{\epsilon_{k,n}^{\tt B2}+\widehat{r}_{k,n}\eta_{k,n}^{\tt B2}}{\big(\widehat{\bE}_n[X_k]\big)^+},
\end{equation} 
\noindent if the condition \eqref{eqn:st-condition} is satisfied with $\lambda = 1.28$ ($\widehat{c}_{k,n}^{\tt B2}=\infty$ otherwise). Then, at stage $n+1$, the following decision is made under {\tt UCB-B2}:
\begin{equation*}\label{eqn:bias-b2}
    I_{n+1}\in\arg\max\limits_k~\Big\{\widehat{r}_{k,n} + \widehat{c}_{k,n}^{\tt B2}  \Big\}.
\end{equation*}

The lack of knowledge for the second-order statistics loosen the upper confidence bound for the rate estimator, which in turn increases the regret. In the following, we provide the regret upper bounds for {\tt UCB-B2} to gain insight about the impact of using variance estimates on the performance of the algorithm.

\begin{theorem}[Regret Upper Bound for {\tt UCB-B2}]\label{thm:ucb-b2u}
Let $\sigma_k$ and $M_k$ be as defined in Theorem \ref{thm:ucb-b1}. Then, we have the following upper bound for the regret under {\tt UCB-B2}:
\begin{equation}
        Reg_{\pi^{\tt B2}}(B) \leq \alpha\sum_{k:\Delta_k>0}\log\Big(\frac{2B}{\mu_*}\Big)(C_k^{\tt B1}+\delta C_k)+O(1),
    \end{equation}
    \noindent where
    \begin{multline}\label{eqn:b2}
    \delta C_k = 21\big(\frac{M_X^4\Delta_k\mu_k}{Var^2(X\inda)}+\frac{Var(X\inda)\Delta_k}{\mu_k}\big).
    \end{multline}
    \noindent for $\mu_k = \bE[X\inda]$.
\end{theorem}

The proof of Theorem \ref{thm:ucb-b2u} involves the analysis of sample variance estimates, and can be found in Appendix F.

\begin{remark}[Impact of Unknown Variances]\normalfont 
The additional terms are caused by the stability of the rate estimator: since we use a variance estimate in the upper confidence bound of $X\ind$, the rate estimator suffers from a longer period of instability, which increases the regret coefficient proportional to $\Delta_k$.
\end{remark}

\subsection{Learning the Correlation: {\tt UCB-B2C}}
Finally we consider the case $(X\ind,R\ind)$ are bounded and correlated, but the second-order moments are unknown. In the absence of correlation, our goal was to estimate $Var(R\inda)$ and $Var(X\inda)$ from the samples of $(X\ind,R\ind)$. When there is a correlation, we have an optimization problem: we need to establish confidence bounds for the LMMSE estimator $\omega_k$ defined in \eqref{eqn:lmmse} as well as the minimum variance $Var(R\inda-\omega_k X\inda)$ by using the samples of $(X\ind,R\ind)$ observed via bandit feedback. We take a loss minimization approach in the statistical learning setting to estimate these quantities.

For any $k\in\bK$, let the empirical LMMSE estimator be defined as follows:
$$\widehat{\omega}_{k,n} = \arg\min\limits_{\omega^\prime\in\mathbb{R}}~\widehat{L}_{k,n}(\omega)$$ where the empirical loss function is the following:
$$\widehat{L}_{k,n}(\omega) = \sum_{i=1}^n\frac{\I\{I_i=k\}}{T_k(n)}\Big( R_i-\widehat{\bE}_n[R]-\omega\big(X_i-\widehat{\bE}_n[X]\big) \Big)^2.$$ It can be shown that $\widehat{\omega}_{k,n}\rightarrow \omega_k$ if $T_k(n)\rightarrow\infty$ as $n\rightarrow\infty$, and moreover the convergence rate is exponential and tight concentration bounds for $\widehat{\omega}_{k,n}$ and $\widehat{L}_{k,n}(\widehat{\omega}_{k,n})$ can be established. Let $M_Z = M_R+\overline{\omega}M_X$ where $\overline{\omega} > \max_k~\omega_k$ is a given parameter, and let 
\begin{align}
    \nu_{k,n}(\omega_k) &= \frac{1.36M_XM_Z}{Var(X\inda)}\sqrt{\frac{\log n^\alpha}{T_k(n)}}, \\
    \nu_{k,n}(L_k) &= M_Z^2\sqrt{\frac{2\log n^\alpha}{T_k(n)}}. 
\end{align}
\noindent Then, it can be shown that $-\widehat{\omega}_{k,n}+\nu_{k,n}(\omega_k)$ and $\widehat{L}_{k,n}(\widehat{\omega}_{k,n})+\nu_{k,n}(\omega_k)$ are high-probability upper bounds for $-\omega_k$ and $\min_\omega~Var(R\inda-\omega X\inda)$, respectively, for large enough $T_k(n)$.

The bias terms in {\tt UCB-B2C} are defined as follows:
\begin{flalign*}
    \epsilon_{k,n}^{\tt B2C}&=\sqrt{\frac{2\widehat{L}_{k,n}(\widehat{\omega}_{k,n})\log(n^\alpha)}{T_k(n)}}+\frac{3M_Z\log(n^\alpha)}{T_k(n)}, \\
    \eta_{k,n}^{\tt B2C}&=\sqrt{\frac{2\widehat{V}_{k,n}(X_k)\log(n^\alpha)}{T_k(n)}}+\frac{3M_X\log(n^\alpha)}{T_k(n)}.
\end{flalign*}
\noindent Then, at stage $n+1$, the following decision is made under {\tt UCB-B2C}:
\begin{equation*}
    I_{n+1}\in\arg\max\limits_k~\Big\{\widehat{r}_{k,n} + \widehat{c}_{k,n}^{\tt B2C}  \Big\},
\end{equation*}
\noindent where $$\widehat{c}_{k,n}^{\tt B2C}=1.4\frac{\epsilon_{k,n}^{\tt B2C}+(\widehat{r}_{k,n}-\widehat{\omega}_{k,n})^+\eta_{k,n}^{\tt B2C}}{\big(\widehat{\bE}_n[X_k]\big)^+},$$ if the stability condition \eqref{eqn:st-condition} is satisfied with $\lambda = 1.28$, and $\widehat{c}_{k,n}^{\tt B2C}=\infty$ otherwise.

In the following, we investigate the impact of using second-order moment estimates on the regret of {\tt UCB-B2C}. The proof can be found in Appendix G.
\begin{theorem}[Regret Upper Bound for {\tt UCB-B2C}]
Let $C_k^{\tt B1}$ be defined as in Theorem \ref{thm:ucb-b1}. Then, we have the following upper bound for the regret under {\tt UCB-B2}:
\begin{equation*}
        Reg_{\pi^{\tt B2C}}(B) \leq \alpha\sum_{k:\Delta_k>0}\log\Big(\frac{2B}{\mu_*}\Big)(C_k^{\tt B1}+\delta C_k^\prime)+O(1),
    \end{equation*}
    \noindent where \begin{equation}\label{eqn:b2c}
     \delta C_k^\prime = \delta C_k + 42\Big(\frac{M_ZM_X}{\sqrt{Var(X\inda)}}+\frac{M_X^4\Delta_k\mu_k}{Var^2(X\inda)}\Big).
    \end{equation}
    \noindent for $\mu_k=\bE[X\inda]$ and $\delta C_k$ defined in \eqref{eqn:b2}.
    \label{thm:ucb-b2c}
\end{theorem}
\noindent Note that the regret of {\tt UCB-B2C} converges to the regret of {\tt UCB-B2}, and they both approach to the performance of the {\tt UCB-B1} Algorithm as $\Delta_k\downarrow 0$.

\section{Conclusions}
In this paper, we considered a very general setting for the budgeted bandit problem where each action incurs a potentially correlated and heavy-tailed cost-reward pair. We proved that positive expected cost and existence of moments of order $2+\gamma$ for some $\gamma>0$ suffice for $O(\log B)$ regret for a given budget $B>0$. For known second-order moments, we proposed two algorithms named {\tt UCB-B1} and {\tt UCB-M1} that exploit the correlation between cost and reward by using an LMMSE estimator. By proposing a regret lower bound, we proved that {\tt UCB-B1} and {\tt UCB-M1} achieve order optimality, and moreover they achieve optimal regret up to a universal constant for the specific case of jointly Gaussian cost and reward pairs, which underlines the significance of second-order moments and correlation in the regret performance. For the case of bounded cost and reward with unknown second-order moments, we proposed learning algorithms {\tt UCB-B2} and {\tt UCB-B2C} that estimate variances as well as LMMSE estimator to approach the performance of {\tt UCB-B1}. We investigated the effect of using these estimates as surrogates in the absence of second-order moments, and showed that they approach the performance of {\tt UCB-B1} in certain cases.


\subsubsection*{Acknowledgements}
This research was supported in part by the NSF grants: CNS-NeTS-1514260, CNS-NeTS-1717045, CMMI-SMOR-1562065, CNS-ICN-WEN-1719371, NSF CCF 1934986, NSF NeTS 1718203, and CNS-SpecEES-1824337; ONR grants: ONR N00014-19-1-2621, ONR N00014-19-1-2566; the DTRA grant HDTRA1-18-1-0050; ARO W911NF-19-1-0379.




\bibliography{sample}

\begin{thebibliography}{34}
\providecommand{\natexlab}[1]{#1}
\providecommand{\url}[1]{\texttt{#1}}
\expandafter\ifx\csname urlstyle\endcsname\relax
  \providecommand{\doi}[1]{doi: #1}\else
  \providecommand{\doi}{doi: \begingroup \urlstyle{rm}\Url}\fi

\bibitem[Agrawal and Devanur(2016)]{agrawal2016linear}
S.~Agrawal and N.~Devanur.
\newblock Linear contextual bandits with knapsacks.
\newblock In \emph{Advances in Neural Information Processing Systems}, pages
  3450--3458, 2016.

\bibitem[Agrawal and Devanur(2014)]{agrawal2014bandits}
S.~Agrawal and N.~R. Devanur.
\newblock Bandits with concave rewards and convex knapsacks.
\newblock In \emph{Proceedings of the fifteenth ACM conference on Economics and
  computation}, pages 989--1006. ACM, 2014.

\bibitem[Asmussen(2008)]{asmussen2008applied}
S.~Asmussen.
\newblock \emph{Applied probability and queues}, volume~51.
\newblock Springer Science \& Business Media, 2008.

\bibitem[Audibert et~al.(2009)Audibert, Munos, and
  Szepesv{\'a}ri]{audibert2009exploration}
J.-Y. Audibert, R.~Munos, and C.~Szepesv{\'a}ri.
\newblock Exploration--exploitation tradeoff using variance estimates in
  multi-armed bandits.
\newblock \emph{Theoretical Computer Science}, 410\penalty0 (19):\penalty0
  1876--1902, 2009.

\bibitem[Badanidiyuru et~al.(2013)Badanidiyuru, Kleinberg, and
  Slivkins]{badanidiyuru2013bandits}
A.~Badanidiyuru, R.~Kleinberg, and A.~Slivkins.
\newblock Bandits with knapsacks.
\newblock In \emph{2013 IEEE 54th Annual Symposium on Foundations of Computer
  Science}, pages 207--216. IEEE, 2013.

\bibitem[Badanidiyuru et~al.(2014)Badanidiyuru, Langford, and
  Slivkins]{badanidiyuru2014resourceful}
A.~Badanidiyuru, J.~Langford, and A.~Slivkins.
\newblock Resourceful contextual bandits.
\newblock In \emph{Conference on Learning Theory}, pages 1109--1134, 2014.

\bibitem[Berry and Fristedt(1985)]{berry1985bandit}
D.~A. Berry and B.~Fristedt.
\newblock Bandit problems: sequential allocation of experiments (monographs on
  statistics and applied probability).
\newblock \emph{London: Chapman and Hall}, 5:\penalty0 71--87, 1985.

\bibitem[Bubeck et~al.(2012)Bubeck, Cesa-Bianchi, et~al.]{bubeck2012regret}
S.~Bubeck, N.~Cesa-Bianchi, et~al.
\newblock Regret analysis of stochastic and nonstochastic multi-armed bandit
  problems.
\newblock \emph{Foundations and Trends{\textregistered} in Machine Learning},
  5\penalty0 (1):\penalty0 1--122, 2012.

\bibitem[Bubeck et~al.(2013)Bubeck, Cesa-Bianchi, and
  Lugosi]{bubeck2013bandits}
S.~Bubeck, N.~Cesa-Bianchi, and G.~Lugosi.
\newblock Bandits with heavy tail.
\newblock \emph{IEEE Transactions on Information Theory}, 59\penalty0
  (11):\penalty0 7711--7717, 2013.

\bibitem[Burkholder(1973)]{burkholder1973distribution}
D.~L. Burkholder.
\newblock Distribution function inequalities for martingales.
\newblock \emph{the Annals of Probability}, pages 19--42, 1973.

\bibitem[Burnetas and Katehakis(1996)]{burnetas1996optimal}
A.~N. Burnetas and M.~N. Katehakis.
\newblock Optimal adaptive policies for sequential allocation problems.
\newblock \emph{Advances in Applied Mathematics}, 17\penalty0 (2):\penalty0
  122--142, 1996.

\bibitem[Cayci et~al.(2019)Cayci, Eryilmaz, and Srikant]{cayci2019learning}
S.~Cayci, A.~Eryilmaz, and R.~Srikant.
\newblock Learning to control renewal processes with bandit feedback.
\newblock \emph{Proceedings of the ACM on Measurement and Analysis of Computing
  Systems}, 3\penalty0 (2):\penalty0 43, 2019.

\bibitem[Combes et~al.(2015)Combes, Jiang, and Srikant]{combes2015bandits}
R.~Combes, C.~Jiang, and R.~Srikant.
\newblock Bandits with budgets: Regret lower bounds and optimal algorithms.
\newblock \emph{ACM SIGMETRICS Performance Evaluation Review}, 43\penalty0
  (1):\penalty0 245--257, 2015.

\bibitem[Dean et~al.(2004)Dean, Goemans, and Vondrdk]{dean2004approximating}
B.~C. Dean, M.~X. Goemans, and J.~Vondrdk.
\newblock Approximating the stochastic knapsack problem: The benefit of
  adaptivity.
\newblock In \emph{45th Annual IEEE Symposium on Foundations of Computer
  Science}, pages 208--217. IEEE, 2004.

\bibitem[Ding et~al.(2013)Ding, Qin, Zhang, and Liu]{ding2013multi}
W.~Ding, T.~Qin, X.-D. Zhang, and T.-Y. Liu.
\newblock Multi-armed bandit with budget constraint and variable costs.
\newblock In \emph{Twenty-Seventh AAAI Conference on Artificial Intelligence},
  2013.

\bibitem[Durrett(2019)]{durrett2019probability}
R.~Durrett.
\newblock \emph{Probability: theory and examples}, volume~49.
\newblock Cambridge university press, 2019.

\bibitem[Guha and Munagala(2009)]{guha2009multi}
S.~Guha and K.~Munagala.
\newblock Multi-armed bandits with metric switching costs.
\newblock In \emph{International Colloquium on Automata, Languages, and
  Programming}, pages 496--507. Springer, 2009.

\bibitem[Gut(2009)]{gut2009stopped}
A.~Gut.
\newblock \emph{Stopped random walks}.
\newblock Springer, 2009.

\bibitem[Gy{\"o}rgy et~al.(2007)Gy{\"o}rgy, Kocsis, Szab{\'o}, and
  Szepesv{\'a}ri]{gyorgy2007continuous}
A.~Gy{\"o}rgy, L.~Kocsis, I.~Szab{\'o}, and C.~Szepesv{\'a}ri.
\newblock Continuous time associative bandit problems.
\newblock In \emph{IJCAI}, pages 830--835, 2007.

\bibitem[Harchol-Balter(2000)]{harchol2000task}
M.~Harchol-Balter.
\newblock Task assignment with unknown duration.
\newblock In \emph{Proceedings 20th IEEE International Conference on
  Distributed Computing Systems}, pages 214--224. IEEE, 2000.

\bibitem[Jelenkovi{\'c} and Tan(2013)]{jelenkovic2013characterizing}
P.~R. Jelenkovi{\'c} and J.~Tan.
\newblock Characterizing heavy-tailed distributions induced by retransmissions.
\newblock \emph{Advances in Applied Probability}, 45\penalty0 (1):\penalty0
  106--138, 2013.

\bibitem[Lai and Robbins(1985)]{lai1985asymptotically}
T.~L. Lai and H.~Robbins.
\newblock Asymptotically efficient adaptive allocation rules.
\newblock \emph{Advances in applied mathematics}, 6\penalty0 (1):\penalty0
  4--22, 1985.

\bibitem[Lalley and Lorden(1986)]{lalley1986control}
S.~Lalley and G.~Lorden.
\newblock A control problem arising in the sequential design of experiments.
\newblock \emph{Annals of probability}, 14\penalty0 (1):\penalty0 136--172,
  1986.

\bibitem[Liu and Zhao(2011)]{liu2011multi}
K.~Liu and Q.~Zhao.
\newblock Multi-armed bandit problems with heavy-tailed reward distributions.
\newblock In \emph{2011 49th Annual Allerton Conference on Communication,
  Control, and Computing (Allerton)}, pages 485--492. IEEE, 2011.

\bibitem[Minsker et~al.(2015)]{minsker2015geometric}
S.~Minsker et~al.
\newblock Geometric median and robust estimation in banach spaces.
\newblock \emph{Bernoulli}, 21\penalty0 (4):\penalty0 2308--2335, 2015.

\bibitem[Nemirovsky and Yudin(1983)]{nemirovsky1983problem}
A.~S. Nemirovsky and D.~B. Yudin.
\newblock Problem complexity and method efficiency in optimization.
\newblock 1983.

\bibitem[Papadimitriou and Tsitsiklis(1999)]{papadimitriou1999complexity}
C.~H. Papadimitriou and J.~N. Tsitsiklis.
\newblock The complexity of optimal queuing network control.
\newblock \emph{Mathematics of Operations Research}, 24\penalty0 (2):\penalty0
  293--305, 1999.

\bibitem[Poor(2013)]{poor2013introduction}
H.~V. Poor.
\newblock \emph{An introduction to signal detection and estimation}.
\newblock Springer Science \& Business Media, 2013.

\bibitem[Robbins(1952)]{robbins1952some}
H.~Robbins.
\newblock Some aspects of the sequential design of experiments.
\newblock \emph{Bulletin of the American Mathematical Society}, 58\penalty0
  (5):\penalty0 527--535, 1952.

\bibitem[Sankararaman and Slivkins(2017)]{sankararaman2017combinatorial}
K.~A. Sankararaman and A.~Slivkins.
\newblock Combinatorial semi-bandits with knapsacks.
\newblock \emph{arXiv preprint arXiv:1705.08110}, 2017.

\bibitem[Siegmund(2013)]{siegmund2013sequential}
D.~Siegmund.
\newblock \emph{Sequential analysis: tests and confidence intervals}.
\newblock Springer Science \& Business Media, 2013.

\bibitem[Tran-Thanh et~al.(2012)Tran-Thanh, Chapman, Rogers, and
  Jennings]{tran2012knapsack}
L.~Tran-Thanh, A.~Chapman, A.~Rogers, and N.~R. Jennings.
\newblock Knapsack based optimal policies for budget--limited multi--armed
  bandits.
\newblock In \emph{Twenty-Sixth AAAI Conference on Artificial Intelligence},
  2012.

\bibitem[Xia et~al.(2015)Xia, Li, Qin, Yu, and Liu]{xia2015thompson}
Y.~Xia, H.~Li, T.~Qin, N.~Yu, and T.-Y. Liu.
\newblock Thompson sampling for budgeted multi-armed bandits.
\newblock In \emph{Twenty-Fourth International Joint Conference on Artificial
  Intelligence}, 2015.

\bibitem[Xia et~al.(2016)Xia, Ding, Zhang, Yu, and Qin]{xia2016budgeted}
Y.~Xia, W.~Ding, X.-D. Zhang, N.~Yu, and T.~Qin.
\newblock Budgeted bandit problems with continuous random costs.
\newblock In \emph{Asian conference on machine learning}, pages 317--332, 2016.

\end{thebibliography}
\bibliographystyle{abbrvnat}

\onecolumn
\appendix
\section{Proof of Proposition \ref{prop:optimality-gap}}
\label{pf:opt-gap}


\begin{proof}
    The proof consists of two parts.
    \begin{enumerate}
        \item In the first part, we find an upper bound for $\bE[{\tt REW}_{\pi^{\tt opt}(B)}(B)]$. In order to achieve this goal, we consider an arbitrary admissible algorithm $\pi\in\Pi$. Since $\pi$ is admissible, we have the following relationship: \begin{equation}\label{eqn:reward-rate-filtration}
        \bE[R_{n,I_n^\pi}|\mathcal{F}_{n-1}^\pi] = r_{I_n}\bE[X_{n,I_n^\pi}|\mathcal{F}_{n-1}^\pi].
        \end{equation}
        \noindent Let $W_t^\pi = \max\limits_{1\leq i \leq t}~S_i^\pi$ for any $t > 0$. Then, inspired by the proof of Wald's equation (see \cite{siegmund2013sequential, xia2015thompson}), we have the following inequality for the expected cumulative reward under $\pi$:
        \begin{align}
            \nonumber \bE[{\tt REW}_\pi(B)] &= \bE\Big[\sum_{i=1}^\infty \I\{W_{i-1}^\pi \leq B\}R_{i,I_i^\pi}\Big], \\ \label{eqn:wald-b} &= \bE\Big[\sum_{i=1}^\infty\bE\big[R_{i, I_i^\pi}|\mathcal{F}_{i-1}^\pi\big]\I\{W_{i-1}^\pi \leq B\}\Big], \\
            &= \label{eqn:wald-c} \bE\Big[\sum_{i=1}^\infty r_{I_i^\pi} \bE\big[X_{i, I_i^\pi}|\mathcal{F}_{i-1}^\pi\big]\I\{W_{i-1}^\pi \leq B\}\Big], \\ \label{eqn:wald-d}
            &\leq r^* \bE\Big[\sum_{i=1}^{N_\pi(B)} X_{i, I_i^\pi}\Big] = r^*\bE\Big[S_{N_\pi(B)}^\pi\Big],
        \end{align}
        \noindent where \eqref{eqn:wald-b} follows since $\pi$ is admissible and $W_{i-1}^\pi\in\mathcal{F}_{i-1}$, and \eqref{eqn:wald-c} follows from the relation \eqref{eqn:reward-rate-filtration} and the fact that $r_{I_i} \leq r^*$ with probability 1. 
        
        

        Note that $S_{N_\pi(B)}^\pi$ is a controlled random walk whose increments $X_{i, I_i^\pi}$ are dependent. Therefore, classical second-order moment results in renewal theory, such as Lorden's inequality \citep{asmussen2008applied}, are not directly applicable to provide an upper bound for $\bE[S_{N_\pi(B)}^\pi]$.  Instead, the following result for the first passage times of submartingales yields a tight upper bound for $\bE[S_{N_\pi(B)}^\pi]$.
        
        \begin{proposition}[\cite{lalley1986control}]\label{prop:lorden}
            Consider a stochastic process $\{(U_n):n \geq 1\}$ with $\bE[U_n] > 0$ adapted to the filtration $\mathcal{F}_n$. Let $S_n = \sum_{i=1}^n U_i$ with $S_0 = 0$ and $N(a) = \inf\{n:S_n > a\}$ be the first passage time of the random walk. 
            
            Assume that there exists constants $\mu_*, \mu^*, \sigma^2>0$ such that $$0<\mu_*\leq \bE[U_n|\mathcal{F}_{n-1}] \leq \mu^* < \infty,$$ and $$Var(U_n|\mathcal{F}_{n-1}) \leq \sigma^2 < \infty,$$ with probability 1 for all $n \geq 1$. If there exists $\gamma > 0$ such that $\bE[(U_n^+)^{2+\gamma}] < \infty$, then there exists a constant $G = G(\mu_*,\mu^*, \sigma^2)$ such that the following holds: $$\bE[S_{N(a)}]-a \leq G,$$ for any $a > 0$.
        \end{proposition}
        \noindent Note that we have $$0 < \min\limits_{k\in[K]}~\bE[X_{1,k}] \leq \bE[X_{i, I_i^\pi}|\mathcal{F}_{i-1}] \leq \max\limits_{k\in[K]}~\bE[X_{1, k}] < \infty,$$ and $$Var(X_{i, I_i^\pi}|\mathcal{F}_{i-1}) \leq \max\limits_{k\in[K]}~Var(X_{1, k}) < \infty,$$ with probability 1 for all $i\geq 1$. Thus, under Assumption \ref{assn:moment-basic}, Proposition \ref{prop:lorden} implies that there exists a constant $G>0$ such that the following holds:
        \begin{equation}\label{eqn:ub-first-passage}
            \bE[S_{N_\pi(B)}^\pi] \leq B + G,
        \end{equation}
        \noindent for all $B > 0$. Hence, \eqref{eqn:wald-d} and \eqref{eqn:ub-first-passage} together imply the following upper bound:
        \begin{equation}\label{eqn:ub-opt}
            \bE[{\tt REW}_{\pi}(B)] \leq r^*(B + G),
        \end{equation}
        \noindent for all $B>0$ and any admissible policy $\pi\in\Pi$. Since the inequality \eqref{eqn:ub-opt} holds for any admissible $\pi\in\Pi$, we have the following result:
        \begin{equation}\label{eqn:ub-reward-opt}
            \bE[{\tt REW}_{\pi^{opt}(B)}(B)] \leq r^*(B+G), ~\forall B > 0.
        \end{equation}
        
        \item In the second part of the proof, we will find a lower bound for $\bE[{\tt REW}_{\pi^*}(B)]$. Since $\pi^*$ is a static policy and $N_{\pi^*}(B)$ is a stopping time, Wald's equation implies the following result \cite{siegmund2013sequential}:
        \begin{equation}\label{eqn:reward-static}
            \bE[{\tt REW}_{\pi^*}(B)] = \bE\big[R_{1, k^*}\big]\bE\big[N_{\pi^*}(B)\big].
        \end{equation}
        For random walks with positive drift, the following inequality holds for any $B>0$ \cite{asmussen2008applied, gut2009stopped}:
        \begin{equation}\label{eqn:lb-first-passage}
            \bE[N_{\pi^*}(B)] \geq \frac{B}{\bE[X_{1,k^*}]}.
        \end{equation}
        \noindent \eqref{eqn:reward-static} and \eqref{eqn:lb-first-passage} together imply the following:
        \begin{equation}\label{eqn:lb-reward-static}
            \bE[{\tt REW}_{\pi^*}(B)] \geq r^*B,~\forall B > 0.
        \end{equation}
    \end{enumerate}
    Inequalities in \eqref{eqn:ub-reward-opt} and \eqref{eqn:lb-reward-static} together imply that the optimality gap is bounded by a constant $G^\star = r^*G$ for all $B>0$. 
\end{proof}
\noindent Proposition \ref{prop:optimality-gap} has a striking implication: the optimality gap is still bounded for unbounded and correlated cost and reward pairs, and this result requires only a mild moment assumption that $\bE[(X_{1,k}^+)^{2+\gamma}],~k\in[K]$ exists for some $\gamma > 0$. Therefore, the simple policy $\pi^*$ serves as a plausible substitute for $\pi^{\tt opt}(B)$, which is NP-hard, for learning purposes.

\section{A Useful Upper Bound for Regret}\label{pf:regret-dec}
The number of trials $N_\pi(B)$ under an admissible policy $\pi$ is a random stopping time, which makes the regret computations difficult. The following proposition, which extends the strategy in \citep{xia2016budgeted} to the case of unbounded and potentially heavy-tailed cost-reward pairs that can take on negative values, provides a useful tool for regret computations.

    \begin{proposition}[Regret Upper Bounds for Admissible Policies]\label{prop:reg-dec}
Suppose that $$\max_k~\bE[|X_{1,k}-\bE[X_{1,k}]|^p] = u_{max} < \infty,$$ for some $p>2$. Let $T_k(n)$ be the number of pulls for arm $k$ in $n$ trials, and $\mu_* = \min_k\bE[X_{1,k}]$. The following upper bound holds for any admissible policy $\pi\in\Pi$ and $B > \mu_*/2$:
\begin{equation}\label{eqn:reg-dec}
    Reg_\pi(B) \leq \sum_k \bE\Big[T_k\Big(\frac{2B}{\mu_*}\Big)\Big]\Delta_k\bE[X_{1,k}] + \frac{\big(\frac{2p^2}{p-1}\big)^pu_{max}}{(2B-\mu_*)^{\frac{p}{2}}\mu_*^{\frac{p}{2}}(\frac{p}{2}-1)}\sum_k \Delta_k\bE[X_{1,k}]+G^\star,
\end{equation}
\noindent where $G^\star=G^\star(\mu_*, \sigma_{max}^2)$ is a constant.
\end{proposition}

The proof of Proposition \ref{prop:reg-dec} relies on a variant of Chebyshev inequality for controlled random walks. Note that $2B/\mu_*$ is a high-probability upper bound for the total number of pulls $N_\pi(B)$, and $\Delta_k\bE[X\inda]$ is the average regret per pull for a suboptimal arm $k$. Proposition \ref{prop:reg-dec} implies that the expected regret after $2B/\mu^*$ pulls is $O(1)$.

\begin{proof}[Proof of Proposition \ref{prop:reg-dec}]
Take an arbitrary admissible policy $\pi\in\Pi$. The regret can be decomposed as follows:
    \begin{align}
        Reg_\pi(B) = \underbrace{\bE[{\tt REW}_{\pi^{opt}(B)}(B)]-\bE[{\tt REW}_{\pi^*}(B)]}_{(a)}+\underbrace{\bE[{\tt REW}_{\pi^*}(B)]-\bE[{\tt REW}_{\pi}(B)]}_{(b)}.
    \end{align}
    \noindent Note that $(a)$ in \eqref{eqn:reg-dec} is the optimality gap for $\pi^*$, which is upper bounded by a constant $G^\star = r^*G$ by Proposition 1. In the following, we provide an upper bound for $(b)$ in \eqref{eqn:reg-dec}.
    
    First, note that the cumulative reward under $\pi^*$ is upper bounded as follows:
    \begin{align}\label{eqn:rewst-ub}
        \nonumber \bE[{\tt REW}_{\pi^*}(B)] &= \bE[N_{\pi^*}(B)]\cdot \bE[R_{1, k^*}],\\
        &\leq Br^* + r^*\frac{\bE[X_{1,k^*}^2]}{\bE[X_{1,k^*}]} = Br^*+c,
    \end{align}
    \noindent where the first line follows from Wald's equation and the second line is a consequence of Lorden's inequality \cite{asmussen2008applied}. Since $B \leq \sum_{i=1}^{N_\pi(B)}X_{i, I_i^\pi}$ under $\pi$, we can further upper bound $\bE[{\tt REW}_{\pi^*}(B)]$ as follows: 
    \begin{align}\label{eqn:rew-st}
        \nonumber \bE[{\tt REW}_{\pi^*}(B)] &\leq \bE\Big[\sum_{i=1}^{N_\pi(B)}r^*X_{i,I_i^\pi}\Big]+r^*\frac{\bE[X_{1,k^*}^2]}{\bE[X_{1,k^*}]}, \\
        &= \bE\Big[\sum_k\sum_{i=1}^\infty \I\{W_{i-1}^\pi \leq B\}\I\{I_i^\pi = k\}r^*\bE[X_{i,k}]\Big]+c.
    \end{align}
    \noindent where $$W_n^\pi =  \max\{S_{1}^\pi,S_{2}^\pi,\ldots,S_{n}^\pi\}.$$ 
    
    Similar to the proof of Proposition 1, we have the following equation for $\bE[{\tt REW}_\pi(B)]$:
    \begin{align}\label{eqn:rew-pi}
        \nonumber \bE[{\tt REW}_\pi(B)] &= \bE\Big[\sum_{i=1}^{N_\pi(B)}R_{i,I_i^\pi}\Big], \\
        &= \bE\Big[\sum_k\sum_{i=1}^\infty \I\{W_{i-1}^\pi \leq B\}\I\{I_i^\pi = k\}r_k\bE[X_{i,k}]\Big]
    \end{align}

\noindent From \eqref{eqn:rew-st} and \eqref{eqn:rew-pi}, we have the following upper bound for $(b)$ in \eqref{eqn:reg-dec}:
\begin{equation}\label{eqn:regret-ub-1}
    \bE[{\tt REW}_{\pi^*}(B)]-\bE[{\tt REW}_\pi(B)] \leq \bE\Big[\sum_k\sum_{i=1}^\infty \I\{W_{i-1}^\pi \leq B\}\I\{I_i^\pi = k\}\Delta_k\bE[X_{i,k}]\Big]+c.
\end{equation}
\noindent For any integer $n_0>1$, the RHS of \eqref{eqn:regret-ub-1} can be upper bounded as follows:
\begin{align}\label{eqn:regret-ub-2}
    \nonumber \bE[{\tt REW}_{\pi^*}(B)]-\bE[{\tt REW}_\pi(B)] &\leq \bE\Big[\sum_{i=1}^{n_0}\sum_k \I\{I_i^\pi = k\}\Delta_k\bE[X_{i,k}]\Big] \\ \nonumber & \quad \hskip 1.5cm + \bE\Big[\sum_{i>n_0} \I\{W_{i-1}^\pi \leq B\}\sum_k \Delta_k\bE[X_{i,k}]\Big]+c, \\ 
    &= \sum_k\bE[T_k^\pi(n_0)]\Delta_k\bE[X_{1, k}] \\\nonumber & \quad \hskip 1.5cm+ \big(\sum_k \Delta_k\bE[X_{i,k}]\big)\sum_{i > n_0}\bP\Big(W_{i-1}^\pi \leq B\Big) + c.
\end{align}

The following martingale-based concentration inequality will be crucial in finding a tight upper bound for the crossing probability of the controlled process $W_n^\pi$ in \eqref{eqn:regret-ub-2}.

\begin{lemma}[Chebyshev Inequality for Submartingales]\label{lem:chebyshev}
Let $\{Z_n:n\geq 0\}$ be a stochastic process adapted to the filtration $\mathcal{F}_n$ such that there exists a pair $(\mu,u)$ satisfying 
\begin{align}\label{eqn:assn-chebyshev}
\begin{aligned}
\bE[Z_n|\mathcal{F}_{n-1}] &\geq \mu > 0, \\ \bE\Big[\big|Z_n-\bE[Z_n|\mathcal{F}_{n-1}]\big|^p|\mathcal{F}_{n-1}\Big] &\leq u < \infty,
\end{aligned}
\end{align} 
\noindent almost surely for all $n\geq 1$ for $p > 2$. Let $S_n = \sum_{i=1}^nZ_i$ and $W_n=\max\limits_{1\leq i \leq n}~S_i$. For a given $B>0$, let $n_0 = \lceil\frac{2B}{\mu}\rceil$. Then we have the following inequality:
\begin{equation}
    \bP(W_{n_0+j} \leq B) \leq \frac{\big(\frac{2p^2}{p-1}\big)^{p}u}{\mu^{p}(n_0+j)^{p/2}}.
\end{equation}
\noindent for all $j\geq 0$.
\end{lemma}

Under an admissible policy $\pi$, the increments $X_{i,I_i^\pi}$ of the controlled random walk $S_n^\pi$ satisfy $\bE[X_{i,I_i^\pi}|\mathcal{F}_{i-1}] \geq \mu_*$ and $\bE\Big[\big|X_{i,I_i^\pi}-\bE[X_{i,I_i^\pi}|\mathcal{F}_{i-1}]\big|^p\Big|\mathcal{F}_{i-1}\Big] \leq u_{max}$ almost surely for all $i$. Therefore, the conditions in \eqref{eqn:assn-chebyshev} are satisfied, and we have:
\begin{equation}
    \bP(W_{n_0+j}^\pi \leq B) \leq \frac{\big(\frac{2p^2}{p-1}\big)^{p}u_{max}}{(2B-\mu_*)^{p/2}\mu_*^{p/2}(n_0+j)^{p/2}}.
\end{equation}
for $n_0=2B/\mu_*,k\geq 1$ and $j\geq 0$. Thus, for $B > \mu_*/2$,
\begin{align}\label{eqn:regret-ub-3}
    \nonumber \sum_{i>n_0}\bP(W_{i-1}^\pi \leq B) &= \sum_{j=0}^{\infty}\bP(W_{n_0+j}^\pi\leq B),\\
    &\leq \frac{\big(\frac{2p^2}{p-1}\big)^{p}u_{max}}{(2B-\mu_*)^{p/2}\mu_*^{p/2}(p/2-1)}.
\end{align}

Substituting $n_0=\frac{2B}{\mu_*}$ and  \eqref{eqn:regret-ub-3} into \eqref{eqn:regret-ub-2} completes the proof.
\end{proof}

\subsection{Proof of Lemma \ref{lem:chebyshev}}
Let $Y_i = Z_i-\bE[Z_i|\mathcal{F}_{i-1}]$ and $M_n = \sum_{i=1}Y_i$, and note that $M_n$ is a martingale. By the assumption \eqref{eqn:assn-chebyshev}, $\mu \leq \bE[Z_i|\mathcal{F}_{i-1}]$ holds almost surely for all $i\geq 1$. Therefore, the following relation holds:
\begin{equation}
    \big\{W_n \leq B\big\} \subset \big\{S_n \leq B\big\} \subset \big\{M_n \leq B-n\mu\big\}.
\end{equation}
Let $n_0 = \frac{2B}{\mu}$. Then, for any $j \geq 0$, we have the following inequality:
\begin{align*}
    \bP(W_{n_0+j} \leq B) &\leq \bP(M_{n_0+j} \leq -\frac{\mu}{2}(n_0+j)),\\
    &\leq \bP\Big(\max\limits_{1\leq i \leq n_0+j}|M_i| > \frac{\mu}{2}(n_0+j)\Big), \\
    &\leq \frac{2^p\bE\Big[\big(\max\limits_{1\leq i \leq n_0+j}|M_i|\big)^p\Big]}{\mu^p(n_0+j)^p}.
\end{align*}
\noindent Then, by $L^p$ maximum inequality for martingales (Theorem 4.4.4 in \citep{durrett2019probability}), we have:
\begin{align}
    \bE\Big[\big(\max\limits_{1\leq i \leq n_0+j}|M_i|\big)^p\Big] &\leq \Big(\frac{p}{p-1}\Big)^p\bE[|M_{n_0+j}|^p].
\end{align}
\noindent For the martingale $M_n$ with increments $\{Y_n:n\geq 1\}$, let $Q_n = Y_1^2+Y_2^2\ldots+Y_n^2$ be the quadratic variation process. It is interesting to note that $M_n$ and $\sqrt{Q_n}$ increase at the same rate in terms of $\mathcal{L}_p$-norm \citep{burkholder1973distribution}:
\begin{equation}
    c_p\bE[|Q_n|^\frac{p}{2}] \leq \bE[|M_n|^p] \leq C_p\bE[|Q_n|^\frac{p}{2}],
\end{equation}
\noindent where $C_p \leq p^p$ and $c_p = 1/C_p$. By H{\"o}lder's inequality, we have the following result for all $i>0$:
$$\bE[|M_{n}|^p] \leq C_pn^{\frac{p}{2}-1}\bE[\sum_{i=1}^{n}|Y_i|^p],$$ for all $n>0$. Given \eqref{eqn:assn-chebyshev}, the following holds: \begin{align} 
\bE[|Y_i|^p]&=\bE\big[\bE[|Y_i|^p\big|\mathcal{F}_{i-1}]\big],\\
&\leq u,
\end{align}
\noindent for any $i\geq 1$. Therefore, we have:
\begin{equation}
    \bP(W_{n_0+j} \leq B) \leq \frac{\big(\frac{2p^2}{p-1}\big)^pu}{\mu^p(n_0+j)^{p/2}}.
\end{equation}

\section{Proof of Theorem \ref{thm:ucb-b1}}\label{pf:ucb-b1}
\begin{proof}
The regret decomposition in Proposition \ref{prop:reg-dec} will be used for the proof. Note that we need to find the expected number of pulls, $\bE[T_k(n)]$, for each arm $k$ with $r_k<r^*$. The following proposition yields an upper bound for $\bE[T_k(n)]$ for any $n > 0$.

\begin{lemma}\label{lem:etn-subgaussian}
Let $\Delta_k = r^*-r_k$ be the reward rate discrepancy and  
\begin{equation}
    \sigma_k^2 =
     \begin{cases} 
      Var(R_{1,k})-\omega_k^2Var(X_{1,k})+(r^*-\omega_k)^2Var(X_{1,k}), & Var(X\inda)\neq 0, \\
      Var(R\inda), & Var(X\inda)=0,
   \end{cases}
\end{equation}
\noindent for all $k\in\bK$, and recall that $\mu_*=\min\limits_k~\bE[X_{1,k}]$. Then we have the following upper bounds for $\bE[T_k(n)]$, the expected number of pulls for arm $k$ in $n$ stages.
\begin{enumerate}
    \item \textbf{Bounded Cost and Reward:} If $\Delta_k>0$ and $|X_{1,k}|\leq M_X$, $|R_{1,k}|\leq M_R$ a.s., then we have the following upper bound under {\tt UCB-B1} with $\alpha > 2$ and $L=2$:
    \begin{equation}\label{eqn:etn-bounded}
        \bE[T_k(n)] \leq 42\log(n^\alpha)\Big(\frac{\sigma_k^2}{\Delta_k^2(\bE[X\inda])^2}+\frac{M_k}{\Delta_k\bE[X\inda]}+\frac{M_X}{\bE[X\inda]}\Big)+12\frac{\alpha}{\alpha-2},
    \end{equation}
    \noindent where $M_k = M_R+r_kM_X$.
    \item \textbf{Jointly Gaussian Cost and Reward:} Let $(X\ind,R\ind)$ be jointly Gaussian with covariance matrix $\Sigma_k$ for all $k$. Then, {\tt UCB-B1} with $\alpha > 2$, $M_X=M_R=0$ and $L=\frac{1}{2}$ yields the following:
    \begin{equation}\label{eqn:etn-gaussian}
       \bE[T_k(n)] \leq 11\log(n^\alpha)\frac{\sigma_k^2}{\Delta_k^2(\bE[X\inda])^2}+12\frac{\alpha}{\alpha-2}.
    \end{equation}
\end{enumerate}
\end{lemma}

The proof then follows from substituting $\bE[T_k(n)]$ in \eqref{eqn:etn-bounded} (or \eqref{eqn:etn-gaussian} for the Gaussian case) into \eqref{eqn:reg-dec}.
\end{proof}
In the rest of this section, we prove Lemma \ref{lem:etn-subgaussian}.
\subsection{Proof of Lemma \ref{lem:etn-subgaussian}}
Consider a suboptimal arm $k$ with $\Delta_k>0$ and a given $n>0$. For any $t < n$, let
$$\hat{c}_{k,t}=\frac{\lambda}{2-\lambda}\frac{\epsilon_{k,n}^{\tt B}+(\widehat{r}_{k,n}-\omega_k)\eta_{k,n}^{\tt B}}{\big(\widehat{\bE}_n[X_k]\big)^+},$$ and
\begin{equation}\label{eqn:ckt}
    c_{k,t} = \frac{\lambda}{\bE[X\inda]}\Big( \frac{2M_k\log(n^\alpha)}{3T_k(t)}+\sqrt{\frac{L\log(n^\alpha)\sigma^2}{T_k(t)}} \Big),
\end{equation}
\noindent where $\sigma^2 = \sqrt{V(X\inda,R\inda)}+(r_k-\omega_k)\sqrt{Var(X\inda)}$ and $\lambda = 1.28$.

We have the following claim based on \citep{audibert2009exploration}.
\begin{claim}\label{claim:ucb-b1}
    Given $n>0$, for any $t< n$, if $I_{t+1} = k$ holds, at least one of the following must be true:
    \begin{itemize}
        \item $E_1 = \{\hat{r}_{k^*,t} + \widehat{c}_{k^*,t} \leq r^*\},$
        \item $E_2 = \{\hat{r}_{k,t} > r_k+ \widehat{c}_{k,t}\},$
        \item $E_3 = \{T_k(t) \leq L\Big(\frac{2\lambda^2}{2-\lambda}\Big)^2\Big( \frac{2\sigma_k^2}{\big(\Delta_k\bE[X\inda]\big)^2} + \frac{M_r}{\Delta_k\bE[X\inda]} \Big)\log(n^\alpha) \},$
        \item $E_4 = \{T_k(t) \leq L\big(\frac{\lambda}{\lambda-1}\big)^2\Big( \frac{Var(X\inda)}{\big(\bE[X\inda]\big)^2} + \frac{M_X}{\bE[X\inda]} \Big)\log(n^\alpha) \},$
    \end{itemize}
\end{claim}
\begin{proof}
    For notational convenience, let $s=T_k(t)$ and $\ell = \log(n^\alpha)$. Suppose to the contrary that neither holds. Then, we have: 
    \begin{equation}\label{eqn:e4}E_4^c \subset \{\frac{2M_X\ell}{3s}+\sqrt{\frac{LVar(X\inda)\ell}{s}} \leq \bE[X\inda]\frac{(\lambda-1)}{\lambda}\},\end{equation}
    \noindent which implies that the rate estimator is stable, thus the concentration inequality in Proposition \ref{prop:rate-est} holds. In order to see \eqref{eqn:e4}, let $x = \frac{\lambda}{\lambda-1}$, $\mu_k=\bE[X\inda]$ and \begin{equation}\label{eqn:u}
        u = Lx^2\Big( \frac{Var(X\inda)}{\big(\bE[X\inda]\big)^2} + \frac{M_X}{\bE[X\inda]} \Big)\ell.
    \end{equation}\noindent Then, for any $s \geq u$, we have the following:
    \begin{align*}
        \frac{2M_X\som}{6x^2\big(M_X\mu_k+Var(X\inda)} + \frac{1}{x}\sqrt{\frac{Var(X\inda)\som}{Var(X\inda)+M_X\mu_k}} \leq \frac{\mu_k}{x},
    \end{align*}
    \noindent since $x > 1$ and $\frac{1-\beta}{3x}+\sqrt{\beta} \leq 1$ for $\beta = \frac{Var(X\inda)}{Var(X\inda)+M_X\mu_k} \in [0,1]$.
    \vskip 0.5cm
    Second, for large $t$, we have the following relation: 
    \begin{equation}\label{eqn:e3}
        E_4^c\cap E_3^c \subset \{\widehat{c}_{k, t} \leq \frac{\Delta_k}{2}\}.
    \end{equation}
    \noindent with high probability. In order to prove \eqref{eqn:e3}, note that the following holds: 
    \begin{equation}
        c_{k,t}\leq \widehat{c}_{k,t} \leq \frac{\lambda}{2-\lambda}{c}_{k,t},
        \label{eqn:hp-bound}
    \end{equation}\noindent with high probability under the event $E_4^c$. Let \begin{equation}\label{eqn:v} v = L\Big(\frac{2\lambda^2}{2-\lambda}\Big)^2\Big( \frac{2\sigma_k^2}{\Delta_k^2\som} + \frac{M_r}{\Delta_k\mu_k} \Big)\ell, \end{equation}\noindent and note that $\sigma^2 \leq 2\sigma_k^2$ by Cauchy-Schwarz inequality. Then, by \eqref{eqn:hp-bound}, for any $s \geq v$, we have:
    \begin{align*}
        \widehat{c}_{k,t} &\leq  \frac{\Delta_k}{2}\Big( \frac{M_r\delmu}{12\lambda\big(2\sigma_k^2+M_r\Delta_k\mu_k\big)} + \sqrt{ \frac{2\sigma_k^2}{2\sigma_k^2+M_r\delmu}} \Big), \\
        &\leq \frac{\Delta_k}{2},
    \end{align*}
    \noindent where the last line holds since $\frac{1-\beta}{12\lambda}+\sqrt{\beta} \leq 1$ for $\lambda > 1$ and $\beta = \frac{2\sigma_k^2}{2\sigma_k^2+M_r\delmu}\in[0,1]$. Since the concentration inequality holds and $E_4^c\cap E_3^c\subset\{\widehat{c}_{k,t} \leq \Delta_k/2\}$, we have:
    \begin{align*}
        \bigcap_{i=1}^4E_i^c \subset \big\{\widehat{r}_{k,t}+\widehat{c}_{k,t} \leq \widehat{r}_{k^*,t}+\widehat{c}_{k^*,t}\Big\},
    \end{align*}
    \noindent which implies that $I_{t+1}=k^* \neq k$.
    \end{proof}
    
    In order to bound $\bP(E_1\cup E_2)$, let $Z\ind = R\ind-\omega_kX\ind$ and 
    \begin{align*}
        \epsilon_{k,t}&=\frac{2M_Z\ell}{3s}+\sqrt{L\frac{ V(X\inda,R\inda)\ell}{s}},\\
        \eta_{k,t}&=\frac{2M_X\ell}{3s}+\sqrt{L\frac{ Var(X\inda)\ell}{s}},
    \end{align*}
    \noindent where $M_Z = M_R+\omega_kM_Z$. Then, the following inequality based on Proposition \ref{prop:rate-est} will be used:
    \begin{align*}
        \bP(|\widehat{r}_{k,t}-r_k| > c_{k, t}) &= \bP\big(\Big|\frac{\widehat{\bE}_t[Z_k]}{\widehat{\bE}_t[X_k]}-\frac{{\bE}[Z_k]}{\bE[X_k]}\Big|>c_{k,t}\big), \\
        &\leq \bP\Big(\Big|\widehat{\bE}_t[Z_k]-\bE[Z_k]\Big|>\epsilon_{k, t}\Big)
        +\bP\Big(\Big|\widehat{\bE}_t[X_k]-\bE[X_k]\Big|>\eta_{k, t}\Big).
    \end{align*}
    \noindent Note that for sub-Gaussian cost and reward pairs, $M_X=M_R=0$ and $L=1/2$ yields Hoeffding's inequality. For the specific case of bounded cost and reward pairs with bounds $M_X$ and $M_R$, respectively, $L=2$ leads to Bernstein's inequality. Using this concentration inequality with \eqref{eqn:hp-bound}, we have the following: $$|\widehat{r}_{k,t}-r_k| > \widehat{c}_{k,t},$$ with high probability. These, along with the union bound, imply the following:
    \begin{equation*}
        \bP\big(E_1\cup E_2\big)\leq \frac{12}{t^{\alpha-1}}.
    \end{equation*}
    \noindent By using this result and Claim \ref{claim:ucb-b1}, we obtain the following inequality:
    \begin{equation*}
        \bE[T_k(n)] \leq u+v+\sum_{t=1}^\infty \frac{12}{t^{\alpha-1}},
    \end{equation*}
    \noindent where $u$ and $v$ are defined in \eqref{eqn:u} and \eqref{eqn:v}, respectively. Choosing $\lambda=1.28$ and substituting $\bE[T_k(n)]$ into Proposition \ref{prop:reg-dec} proves the result.

\section{Proof of Theorem \ref{thm:ucb-m1}}
For any $k$, if $X\ind$ or $R\ind$ has heavy tails, then the empirical rate estimator is weak in the sense that the convergence rate is polynomial rather than exponential \citep{bubeck2013bandits}. In the following, we propose a median-based rate estimator, and prove that it is robust in the sense that an exponential convergence rate is achieved even if the cost and reward are heavy-tailed. The correlation between $X\inda$ and $R\inda$ is exploited for improved coefficients.
\begin{proposition}[Median-based rate estimation]\label{prop:m1}
    For any given $\delta \in (0, 1)$, let $$m = \lceil 3.5\log(\delta\inv)\rceil + 1,$$ and $G_1,G_2,\ldots, G_m$ be a partition of $[s]$ where $|G_j|=\lfloor \frac{s}{m}\rfloor$ for each $j$. Define $\widehat{\bE}_{G_j}[X_k]$ (and $\widehat{\bE}_{G_j}[R_k]$) be the sample mean of $X\ind$ (and $R\ind$) in partition $G_j$, and $\tilde{r}_{j,k}=\frac{\widehat{\bE}_{G_j}[R_k]}{\widehat{\bE}_{G_j}[X_k]}$ for each $j$. Given $\lambda > 1$, if 
    \begin{equation}\label{eqn:stability-m1}
    s \geq 135\Big(\frac{\lambda}{\lambda-1}\Big)^2Var(X\inda)\log(1.4\delta\inv),\end{equation}\noindent then the following inequality holds: $$\bP\Big(\big|\overline{r}_{s,k}-r_k\big| > \frac{22\lambda}{\bE[X\inda]}\sqrt{\frac{\sigma_k^2\log(\delta\inv)}{s}}\Big)\leq 1.4\delta,$$ where $\overline{r}_{s,k}=\underset{{1\leq i \leq m}}{median}~\tilde{r}_{j,k}$ and $\sigma_k$ is defined in \eqref{eqn:sigma-k} .
\end{proposition}

\begin{proof}
Given $\lambda > 1$, for any $j\in[m]$ and $p\in(0,\frac{1}{2})$, if $$\sqrt{\frac{4mVar(X\inda)}{sp}}\leq \frac{\bE[X\inda](\lambda-1)}{\lambda},$$ we have the following:
\begin{align*}
    \bP(|\tilde{r}_{j,k}-r_k|>\frac{\lambda}{\bE[X\inda]}\sqrt{\frac{8m\sigma_k^2}{sp}})\leq p,
\end{align*}
\noindent by Chebyshev's inequality and Proposition \ref{prop:rate-est}. Therefore, by Theorem 3.1 in \citep{minsker2015geometric}, we have:
$$\bP\Big(|\overline{r}_{s,k}-r_k| > \frac{1-\beta}{\sqrt{1-2\beta}}\frac{\lambda}{\bE[X\inda]}\sqrt{\frac{8m\sigma_k^2}{sp}}\Big) \leq e^{-m\psi(\beta; p)},$$ for $\beta\in(p,\frac{1}{2})$ and $$\psi(\beta;p)=\beta\log\Big(\frac{\beta}{p}\Big)+(1-\beta)\log\Big(\frac{1-\beta}{1-p}\Big).$$
For a given $\delta \in (0,1)$, the values $m = \lfloor 3.5\log(\delta\inv)\rfloor+1$, $\beta=8/17$ and $p=0.1$ yield the result.
\end{proof}

The proof of Theorem \ref{thm:ucb-m1} is based on the regret decomposition in Appendix \ref{pf:regret-dec} and the following lemma.
\begin{lemma}\label{lem:etn-ht}
    For any $\lambda > 1$ and $\alpha > 2$, we have:
    \begin{equation}
        \bE[T_k(n)] \leq \log(n^\alpha)\Big(\frac{484\lambda^2\sigma_k^2}{\Delta_k^2(\bE[X\inda])^2}+\frac{135(\frac{\lambda}{\lambda-1})^2Var(X\inda)}{(\bE[X\inda])^2}\Big)+48\frac{\alpha}{\alpha-2},
    \end{equation}
    \noindent for any $k$ that satisfies $r_k<r^*$.
\end{lemma}

Lemma \ref{lem:etn-ht} is proved in an identical way to Lemma \ref{lem:etn-subgaussian} by using the concentration inequality proposed in Proposition \ref{prop:m1}.

\section{Proof of Theorem \ref{thm:regret-lb}}\label{pf:regret-lb}
\begin{proof}
The regret under any admissible policy can be lower bounded as follows:
\begin{lemma}\label{lem:regret-lb-dec}
    For any $B>0$, let $$\phi_\pi(B)=\sum_k\bE[\I\{I_{N_\pi(B)}=k\}]\bE[X_{N_\pi(B),k}],$$ be the average cost in the last epoch under an admissible policy $\pi$, $\mu_+ = \max\limits_k~\bE[X\inda^+]$ and $\mu_* = \min\limits_k~\bE[X\inda]$. Then, the regret under $\pi$ is lower bounded as follows:
    \begin{equation}\label{eqn:reg-lb-1}
        Reg_\pi(B) \geq \sum_k\Delta_k\bE[X\inda]\bE[T_k(\big\lceil\sqrt{2B/\mu_*}\big\rceil)]-\frac{\mu_+}{\mu_*}(1+\frac{1}{\sqrt{2B}})\sum_k\Delta_k\bE[X\inda]-\phi_\pi(B).
    \end{equation}
\end{lemma}

Then, under the conditions stated in Theorem \ref{thm:regret-lb}, the following result provides an asymptotic lower bound for $\bE[T_k(n)]$ for any $k$ with $r_k<r^*$.

\begin{lemma}\label{lem:lb-burnetas}
    If $\pi\in\Pi$ is a policy such that $\bE[T_k^\pi(n)] = o(n^\alpha)$ for any $\alpha > 0$ and $k$ such that $r_k(\theta_k) < r^*$, then we have the following lower bound:
\begin{equation}\label{eqn:reg-lb-2}
    \underset{n\rightarrow\infty}{\lim\inf}~\frac{\bE[T_k(n)]}{\log(n)} \geq \frac{1}{D_k^\star},
\end{equation}
\noindent where $D_k^\star$ is the solution to the following optimization problem:
\begin{equation*}
D_k^\star = \min_{\theta\in\Theta_k} D(P_{k,\theta_k}||P_{k,\theta})
\mbox{ subject to } r_k(\theta) \geq r^*.
\end{equation*}
\end{lemma}
    Lemma \ref{lem:lb-burnetas} can be proved by a straightforward adaptation of Theorem 1 in \citep{burnetas1996optimal}.

If the moment condition $\bE[(X_{1,k})^{2+\gamma}]<\infty$ holds for all $k$, then the term $\phi_\pi(B)=O(1)$ as $B\rightarrow\infty$ by Lorden's inequality \citep{asmussen2008applied}. Therefore, using \eqref{eqn:reg-lb-1} and \eqref{eqn:reg-lb-2}, we obtain the result.
\end{proof}
\subsection{Proof of Lemma \ref{lem:regret-lb-dec}}

Take any admissible policy $\pi$ and $B>0$. We have the following inequalities:
\begin{align*}
    Reg_\pi(B)&=\bE[{\tt REW}_{\pi^{\tt opt}(B)}(B)]-\bE[{\tt REW}_\pi(B)],\\
    &\geq \bE[{\tt REW}_{\pi^*}(B)]-\bE[{\tt REW}_\pi(B)],
\end{align*}
\noindent since $\bE[{\tt REW}_{\pi^{\tt opt}(B)}(B)] \geq \bE[{\tt REW}_{\pi^*}(B)]$ by definition. Then, by using a similar decomposition as \eqref{eqn:rew-st}, we have the following:
\begin{align}\label{eqn:regret-lb-dec}
    Reg_\pi(B) &\geq \bE[\sum_{t=1}^\infty\sum_k\Delta_k\bE[X\inda]\I\{W_{t-1}\leq B\}\I\{I_t=k\}]-r^*\phi_\pi(B), \\
    &\geq \bE[\sum_{t=1}^{n_0}\sum_k\Delta_k\bE[X\inda]\I\{W_{t-1}\leq B\}\I\{I_t=k\}]-r^*\phi_\pi(B)
\end{align}
\noindent for any $n_0>0$, where $W_t^\pi=\max\limits_{1\leq i \leq t}~S_i^\pi$. Since $\I\{W_{t-1}^\pi\leq B\}=1-\I\{W_{t-1}^\pi> B\}$, we have:
\begin{align}\label{eqn:regret-lb-dec-2}
    Reg_\pi(B)\geq \sum_k\bE[T_k(n_0)]\Delta_k\bE[X\inda]-(\sum_k\Delta_k\bE[X\inda])\sum_{t=1}^{n_0}\bP(W_{t-1}^\pi > B)-r^*\phi_\pi(B).
\end{align}
\noindent We have the following result: 
\begin{align} \label{eqn:prob-wt}
    \begin{aligned}
    \bP(W_t^\pi > B) &\leq \bP(\max\limits_{1\leq i \leq t}~(S_i^\pi)^+ > B), \\
    &\leq \frac{\bE[(S_t^\pi)^+]}{B}, \\
    &\leq \frac{\bE[\sum_{i=1}^tX_{i,I_i}^+]}{B}\leq \frac{t\mu_+}{B},
    \end{aligned}
\end{align}
\noindent where the second inequality follows from Doob's martingale inequality \citep{durrett2019probability}, and the last inequality is true since $\mu_+ \geq X_{i,I_i}^+$ with probability 1 for all $i$. Substituting \eqref{eqn:prob-wt} into \eqref{eqn:regret-lb-dec-2}, and setting $n_0 = \sqrt{2B/\mu_*}$ yields the result.

\section{Proof of Theorem \ref{thm:ucb-b2u}}\label{pf:ucb-b2u}
In the design of {\tt UCB-B2}, empirical variance estimates are used, which require a modified analysis compared to {\tt UCB-B1}.
\begin{lemma}
If $\Delta_k>0$ and $|X_{1,k}|\leq M_X$, $|R_{1,k}|\leq M_R$ a.s., then we have the following upper bound under {\tt UCB-B2} with $\alpha > s$:
    \begin{multline}\label{eqn:etn-emp-ber}
        \bE[T_k(n)] \leq 21\log(n^\alpha)\Big(\frac{M_X^4}{Var^2(X\inda)} + \frac{2M_X}{\bE[X\inda]} + \frac{3Var(X\inda)}{\bE^2[X\inda]}\Big)\\+42\log(n^\alpha) \Big(\frac{\sigma_k^2}{\Delta_k^2(\bE[X\inda])^2}+\frac{M_k}{\Delta_k\bE[X\inda]}\Big)+48\frac{\alpha}{\alpha-2},
    \end{multline}
    \noindent where $\sigma_k = Var(R\inda)-\omega_k^2Var(X\inda)$ and $M_k = M_R+r_kM_X$.
\end{lemma}

\begin{proof}

The proof follows along the same lines as Theorem \ref{thm:ucb-b1} and the proof of Theorem 3 in \citep{audibert2009exploration}. For any $k$, let the variance estimate $\widehat{V}_{k,n}(X_k)$ be defined as follows: $$\widehat{V}_{k,n}(X_k) = \frac{1}{T_k(n)}\sum_{i=1}^n\I\{I_i=k\}\big(X_{i,k}-\widehat{\bE}_n[X\inda]\big)^2,$$ where $\widehat{\bE}_n[X_k]$ is the empirical mean of the observations up to epoch $n$. Also, let $\nu_{k,n}$ be defined for $X_k\in[0,M_X]$ as follows: $$\nu_{k,n}(X_k) = M_X^2\Big(\frac{7\log(n^\alpha)}{6T_k(n)}+\sqrt{\frac{\log(n^\alpha)}{2T_k(n)}}\Big),~\alpha > 2.$$ Then, it can be shown by using Bernstein's inequality that $\widehat{V}_{k,n}(X_k)+\nu_{k,n}(X_k)$ is an upper bound for $Var(X\inda)$ with high probability. Using this result, we obtain the sample size required for the stability of the rate estimator by using identical steps as Theorem \ref{thm:ucb-b1}.
\end{proof}

\section{Proof of Theorem \ref{thm:ucb-b2c}}
The proof of Theorem \ref{thm:ucb-b2u} follows the same steps as Theorem \ref{thm:ucb-b2c}, with the difference that the correlation between $X\ind$ and $R\ind$ are estimated in the latter. In order to observe the effect of using LMMSE estimates to exploit correlation, we first present concentration bounds for $\omega_k$ and $\min\limits_\omega~Var(R\inda-\omega X\inda)$.

\subsection{Preliminaries}
Throughout this subsection, we consider a generic iid stochastic process $(X_n,R_n)$ with $X_n\in[0,M_X]$ and $R_n\in[0,M_R]$. For this process, let $\omega_* = \arg\min_\omega~L(\omega)$ where $$L(\omega)=Var(R_1-\omega X_1),$$ and $\widehat{\omega}_s = \arg\min_\omega~\widehat{L}_s(\omega)$ where $$\widehat{L}_s(\omega) = \frac{1}{s}\sum_{i=1}^s\Big(R_i-\widehat{\bE}_s[R]-\omega(X_i-\widehat{\bE}_s[X])\Big)^2.$$ Note that $\omega_* = \frac{Cov(X_1,R_1)}{Var(X_1)}$ and $\widehat{\omega}_s = \frac{\widehat{Cov}_s(X,R)}{\widehat{Var}_s(X)}$ where $$\widehat{Cov}_s(X,R) = \frac{1}{s}\sum_{i=1}^s(R_i-\widehat{\bE}_s[R])(X_i-\widehat{\bE}_s[X]),$$ is the empirical covariance and $\widehat{Var}_s(X) = \widehat{Cov}_s(X,X)$. In the following, we propose concentration inequalities for $\omega_*$ and $L(\omega_*)$.

\begin{proposition}[Concentration of LMMSE Estimator]\label{prop:lmmse}
    Let $M_Z\geq M_R+\omega_*M_X$ and $\lambda = 1+\frac{1}{2\sqrt{2}}$. Then, for any $\delta\in(0, 1)$, if 
    \begin{equation}\label{eqn:stability-omega}
    s \geq \frac{63M_X^4\log(\delta\inv)}{Var^2(X_1)},\end{equation} \noindent then the following inequalities hold simultaneously:
    \begin{align*}
    \bP(|\omega_*-\widehat{\omega}_s| > \frac{\lambda M_Z M_X}{Var(X_1)}\sqrt{\frac{\log(\delta\inv)}{s}}) \leq 12\delta,\\
    \bP(|L(\omega_*)-\widehat{L}_s(\widehat{\omega}_s)| >  M_Z^2\sqrt{\frac{2\log(\delta\inv)}{s}}) \leq 18\delta.
    \end{align*}
\end{proposition}

\begin{proof}
    For the first inequality, recall that $\omega_* = \frac{Cov(X_1,R_1)}{Var(X_1)}$ and $\widehat{\omega}_s$ is the ratio of empirical estimates for $Cov(X_1,R_1)$ and $Var(X_1)$. Therefore, we can use Proposition \ref{prop:rate-est} for the proof. Note that \eqref{eqn:stability-omega} is the stability condition for the estimator $\widehat{\omega}_s$. Since $s \geq \frac{1}{2}\log(\delta\inv)$, Hoeffding's inequality yields the following result for the empirical covariance:
    \begin{equation}\label{eqn:hoeffding-var}
        \bP(|\widehat{Cov}_s(X_1, R_1)-Cov(X_1,R_1)| > M_XM_R\sqrt{\frac{\log(\delta\inv)}{s}}) \leq 6\delta.
    \end{equation}
    \noindent Using this twice for $\widehat{Cov}_s(X_1, R_1)$ and $\widehat{Var}_s(X_1)$, we obtain the first inequality.
    
    For the second inequality, first we make the following decomposition:
    \begin{align}\label{eqn:lmmse-dec}
        |\widehat{L}_s(\widehat{\omega}_s)-L(\omega_*)| = |\widehat{L}_s(\omega_*)-L(\omega_*)| + |\widehat{L}_s(\widehat{\omega}_s)-\widehat{L}_s({\omega}_*)|.
    \end{align}
    \noindent For the first term on the RHS of \eqref{eqn:lmmse-dec}, we have the following result: $$|\widehat{L}_s(\omega_*)-L(\omega_*)| \leq M_Z^2\sqrt{\frac{\log(\delta\inv)}{s}},$$
    by applying Hoeffding's inequality for the variance \eqref{eqn:hoeffding-var} to the decomposition: $$Var(R_1-\omega X_1) = Var(R_1)+\omega^2Var(X_1)-2Cov(X_1,R_1),$$ and its empirical counterpart. For the second term on the RHS of \eqref{eqn:lmmse-dec}, note that the following identity holds by the orthogonality principle:
    \begin{equation}
        \widehat{L}_s(\omega)=\widehat{L}_s(\widehat{\omega}_s)+|\omega-\widehat{\omega}_s|^2\widehat{Var}_s(X_1),
    \end{equation}
    \noindent for any $\omega\in\mathbb{R}$. Therefore, by union bound, we have the following result:
    $$\bP\Big(|{L}_s(\omega_*)-\widehat{L}_s(\widehat{\omega}_s)| > M_Z^2\Big(\sqrt{\frac{\log(\delta\inv)}{s}}+\frac{3\lambda^2M_X^2\log(\delta\inv)}{2Var(X_1)s}\Big)\Big)\leq 18\delta,$$ from the concentration result for $|\omega_*-\widehat{\omega}_s|$ and \eqref{eqn:hoeffding-var} with $M_X^2\sqrt{\frac{\log(\delta\inv)}{s}}\leq\frac{Var(X_1)}{2}$ by \eqref{eqn:stability-omega}. Since $s$ is assumed to be sufficiently large by \eqref{eqn:stability-omega}, we have: $$\sqrt{\frac{\log(\delta\inv)}{s}}>\frac{3\lambda^2M_X^2\log(\delta\inv)}{2Var(X_1)s},$$ which concludes the proof.
\end{proof}

\subsection{Proof of Theorem \ref{thm:ucb-b2c}}
The proof follows a similar steps as the proof of Theorem \ref{thm:ucb-b2u} (see Appendix \ref{pf:ucb-b2u}). The main difference is the use of LMMSE estimator as a surrogate for $V(X\inda,R\inda)$. By using Proposition \ref{prop:lmmse}, one can show the following:
    \begin{multline*}\label{eqn:etn-emp-berc}
        \bE[T_k(n)] \leq 21\log(n^\alpha)\Big(\frac{3M_X^4}{Var^2(X\inda)} + \frac{2M_X}{\bE[X\inda]} + \frac{3Var(X\inda)}{\bE^2[X\inda]}\Big)\\+42\log(n^\alpha) \Big(\frac{\sigma_k^2}{\Delta_k^2(\bE[X\inda])^2}+\frac{M_k+M}{\Delta_k\bE[X\inda]}\Big)+64\frac{\alpha}{\alpha-2},
    \end{multline*}
    \noindent where $M=\frac{M_XM_Z}{\sqrt{Var(X\inda)}}$, $\sigma_k = Var(R\inda)-\omega_k^2Var(X\inda)$ and $M_k = M_R+r_kM_X$.



\end{document}